\documentclass{ecai} 

\usepackage{latexsym}
\usepackage{amssymb}
\usepackage{amsmath}
\usepackage{amsthm}
\usepackage{booktabs}
\usepackage{enumitem}
\usepackage{graphicx}
\usepackage{color}
\usepackage{graphicx} 
\usepackage{amsfonts}
\usepackage{algorithm}
\usepackage{algorithmicx}
\usepackage[noend]{algpseudocode}
\usepackage{array} 
\usepackage{makecell}
\usepackage[percent]{overpic}
\usepackage{subcaption}

\newtheorem{theorem}{Theorem}

\newtheorem{proposition}[theorem]{Proposition}

\newtheorem{definition}{Definition}
\newtheorem{example}{Example}

\newcommand{\BibTeX}{B\kern-.05em{\sc i\kern-.025em b}\kern-.08em\TeX}

\begin{document}


\begin{frontmatter}


\paperid{1253}


\title{Count-based Novelty Exploration in Classical Planning}


\author[A]{\fnms{Giacomo}~\snm{Rosa}\thanks{Corresponding Author. Email: rosag@student.unimelb.edu.au \\\\ This is the extended version of a paper accepted for publication at the \textit{27th European Conference on Artificial Intelligence (ECAI)}, Santiago de Compostela, Spain, 2024.}}
\author[A]{\fnms{Nir}~\snm{Lipovetzky}}

\address[A]{The University of Melbourne, Australia}


\begin{abstract}
Count-based exploration methods are widely employed to improve the exploratory behavior of learning agents over sequential decision problems. Meanwhile, Novelty search has achieved success in Classical Planning through recording of the first, but not successive, occurrences of tuples. In order to structure the exploration, however, the number of tuples considered needs to grow exponentially as the search progresses. We propose a new novelty technique, classical count-based novelty, which aims to explore the state space with a constant number of tuples, by leveraging the frequency of each tuple's appearance in a search tree. We then justify the mechanisms through which lower tuple counts lead the search towards novel tuples. We also introduce algorithmic contributions in the form of a trimmed open list that maintains a constant size by pruning nodes with bad novelty values. These techniques are shown to complement existing novelty heuristics when integrated in a classical solver, achieving competitive results in challenging benchmarks from recent International Planning Competitions. Moreover, adapting our solver as the frontend planner in dual configurations that utilize both memory and time thresholds demonstrates a significant increase in instance coverage, surpassing current state-of-the-art solvers.
\end{abstract}

\end{frontmatter}


\section{Introduction}

Research on \textit{width-based search methods} \citep{lipovetzky2012width} has had a significant impact in planning, over the past decade, through the introduction of search algorithms which rely on \textit{novelty} heuristics to induce an efficient exploration of the state-space. Novelty metrics achieve this by comparing a state's information content with that of states visited in the past. Width-based algorithms adopting Novelty alongside traditional heuristics have been central to improving state-of-the-art results in Classical Planning in recent years \citep{lipovetzky2021planning}, with search performance often being attributed to a balance between exploration and exploitation, where Novelty drives the exploration while traditional heuristics direct exploitation. This does not come without limitations, as \citet{lipovetzky2012width,lipovetzky2014width} show that the complexity of computing novelty metrics needed to solve planning problems is exponential in their cardinality. In practice, this causes novelty metrics of cardinality greater than 2 to be computationally unfeasible, limiting the technique's effectiveness in domains that would benefit from a higher cardinality. The cardinality is connected with a hardness measure for Classical Planning known as \textit{classical planning atomic width}. Multiple contributions have sought to address this limitation. \citet{lipovetzky2017best} introduce \textit{partition functions}, which subdivide planning problems into smaller sub-problems through the use of \textit{partitioning heuristics} to control the direction of search and increase the number of novel nodes. \citet{katz2017adapting} provide a definition of novelty of a state with respect to its heuristic estimate, providing multiple novelty measures which quantify the novelty degree of a state in terms of the number of novel and non-novel state facts. More recently, \citet{singh2021approximate} introduce \textit{approximate novelty}, which uses an approximate measurement of state novelty which is more time and memory efficient, proving capable of estimating novelty values of cardinality greater than 2 in practical scenarios.
Relating Novelty with other concepts, such as \textit{dominance pruning}, also constitutes an active area of research \citep{dold2024novelty,gross2020novel}. 

All mentioned techniques limit themselves to the original idea of measuring state information content through the occurrence of tuples in the search history. Instead, we propose a count-based measure of state novelty, \textit{classical count-based novelty}, which seeks to induce efficient exploration of the state space by making use of the additional information contained in the count of occurrences of tuples in the search history. This addresses shortcomings of the current Novelty framework (see \citep{lipovetzky2021planning}), which we refer to as \textit{width-novelty} to distinguish from our contributions in this paper. Our proposed count-based metric is not limited by width-novelty's binary classification of novel information, providing a more fine-tuned separation of the degree of novelty of a state and maintaining its informedness without the risk of exhausting novel nodes. A key motivation behind our study is thus to obtain a more general novelty framework that can maintain its efficacy across diverse sets of problems in Classical Planning, such as domains that require higher atomic widths.

In this regard, we note that count-based exploration techniques are well studied in relation to the exploration-exploitation problem in Multi-Arm Bandits and Reinforcement Learning (RL) settings. Such algorithms record state visitation to obtain an \textit{exploration bonus} used to guide the agent towards a more efficient exploration of the state-space, where algorithms such as MBIE-EB \citep{strehl2008analysis} achieve theoretical bounds on sample complexity in tabular settings. The focus of our research diverges from these methods, as we aim to discover a heuristic to control the order of state exploration in a Classical Planning context. Instead of state counts, we base our approach on the frequency of tuple events, inspired by work on width-novelty in the field of Classical Planning \citep{lipovetzky2012width, lipovetzky2014width, lipovetzky2017polynomial}. Still, our contributions provide a useful basis to connect count-based exploration across the two fields.

We also introduce algorithmic contributions in the form of a simple memory-efficient open list designed with count-based novelty in mind. Polynomial width-based planning algorithms prune nodes whose novelty cardinality is worse than a given bound to achieve a more efficient search \citep{lipovetzky2017polynomial}. Inspired on this idea, our contribution allows us to prune nodes with bad novelty values with a gradual and self-balancing cutoff without maintaining an explicit threshold value. 

Finally, we demonstrate the effectiveness of our proposed planning algorithms as fast but memory-intensive \textit{frontend} solvers through an effective use of a \textit{memory threshold}, which allows us to relate the progress of search to the amount of information we store from the history of a search. Many successful solvers such as FF, Probe or Dual-BFWS rely on such dual strategy \citep{hoffmann2001ff,lipovetzky2011searching,lipovetzky2017best}, with the frontend of such solvers playing a key role in their performance. The performance of our proposed frontend planner could improve such solvers even further.

Our contributions consist of a new novelty technique, with theoretical analysis tying it to existing width-novelty measures, trimmed open list, a planner, $BFNoS$, which integrates these techniques, and a procedure to adapt BFNoS as an effective frontend planner in a dual strategy. We structure our paper as follows. In section \ref{Count Based Novelty} we introduce classical count-based novelty metrics for Classical Planning and provide related theoretical findings. We then propose a novel open list implementation to exploit classical count-based novelty more efficiently. Section \ref{Experiments} is divided into two components: we first compare the performance of solvers incorporating our proposed techniques, and then show the impact of our frontend solver when used in conjunction with different time-and-memory thresholds and backend solvers, providing state-of-the-art performance.

\section{Background}
\paragraph{Classical Planning}
The classical planning model is defined as $S=\langle S, s_0, S_G, A, f \rangle$, where $S$ is a discrete finite state space, $s_0$ is the initial state, $S_G$ is the set of goal states, and $A(s)$ denotes the set of actions $a \in A$ that deterministically map one state $s$ into another $s'= f(a,s)$, where $A(s)$ is the set of actions applicable in $s$. We adopt a notation whereby, in a classical planning problem, a state is visited (generated) sequentially at each time-step $t$. Let \(s_t\in S\) denote the \(t^{th}\) visited (generated) state in a search problem. We use \(s_{0:t}\) to denote the sequence of $t+1$ states generated at time-steps $0, 1, ..., t$.
A solution to a classical planning model is given by a plan, a sequence of actions $a_{0}, ..., a_{x_m}$ that induces a state sequence $s_{0:x_{m+1}}$ such that $a_{x_i}\in A(s_{x_i})$, $s_{x_{i+1}}=f(a_{x_i},s_{x_i})$, and $s_{x_{m+1}}\in S_G$.

We use STRIPS planning language \citep{haslum2019introduction} to define a classical planning problem $P=\langle F,O,I,G \rangle$, where $F$ denotes the set of boolean variables, $O$ denotes the set of operators, $I\subseteq F$ is the set of atoms that fully describe the initial state, and $G\subseteq F$ is the set of atoms present in the goal state. An optimal plan consists of the shortest possible solution to a given problem $P$. In this research, we look at \textit{satisficing} planners, that is, planners which are not constrained to searching for optimal plans, but rather aim for computing good-quality plans fast.

\paragraph{Width-Based Search}
\textit{Best-First Width Search} (BFWS) \citep{lipovetzky2017polynomial} refers to a family of planners which adopt a \textit{greedy best-first search} algorithm, using a novelty measure as first heuristic. A greedy best first search planner is a planner which visits nodes in the order specified solely by an \textit{evaluation function $h$}, potentially breaking ties through the use of secondary heuristics. The main peculiarity of using a primary novelty heuristic comes from the fact that it is goal-unaware, thus prioritizing an efficient exploration of the state space over seeking states which are expected to be closer to the goal. The search is then directed to the goal through the use of secondary tie-breaking heuristics as well as \textit{partition functions}, that is, evaluation functions $h$ used to partition the set of states considered in the computation of novelty measures into disjoint subsets, ignoring occurrences of variables in states belonging to separate subsets.

\paragraph{Count-based exploration}
Count-based exploration methods have been studied to address the exploration-exploitation dilemma inherent in learning algorithms by allowing agents to prioritize actions that lead to states with uncertain or unexplored dynamics, thereby facilitating more effective learning of the environment's structure. This is often achieved by incorporating an exploration bonus added to the agent's reward upon visiting a state, encouraging the exploration of states with low visitation counts. Among the best-known examples is the UCB1 bandit algorithm \citep{auer2002finite}, which 
performs a near-optimal balancing of exploration and exploitation in the stateless multi-armed bandit problem. This is achieved by selecting actions, referred to as the arms of a bandit, which maximize an \textit{upper confidence bound}, the sum of the empirical average rewards $Q_t(i)$ of selecting arm $i$, and a confidence interval term $\sqrt{\frac{2\log N}{N(i)}}$, where $N(i)$ is the count of pulls of arm $i$, and $N$ is count of total arm pulls.

\section{Count Based Novelty} \label{Count Based Novelty}

\textit{Classical count-based novelty} operates over states that assign a value to a finite number of variables \(v\in V\) over finite and discrete domains. In problems defined via STRIPS, without loss of generality, $V = F$ are boolean variables. Let \(V\) be the set of all variables, and \(U^{(k)} = \{X \subseteq V \mid |X| = k\}\) the set of all \(k\)-element variable conjunctions. A tuple \(u \in U^{(k)}\), specifically \(u = \{v_1, v_2, \ldots, v_k\}\), represents a conjunction of \(k\) variables. Given a state \(s\) that assigns a boolean value to each variable in \(V\), the value of the tuple \(u\) in state \(s\), denoted \(s(u)\), is defined as the conjunction of the values of the $k$ variables in \(u\), \(s(u) = s(v_1) \land s(v_2) \land \cdots \land s(v_k)\), where \(s(v_i)\) is the value of variable \(v_i\) in state \(s\). We say \(s(u)\) is \textit{true} if all $v \in s(u)$ are \textit{true}, and tuple $u$ is \textit{true} in $s$ if $s(u)$ is \textit{true}. Let \(s_{0:t}(u)\) denote the sequence of values of tuple \(u\in U^{(k)}\) in state sequence \(s_{0:t}\), and let \(U^{+(k)}(s) \subseteq U^{(k)}\) denote the set of tuples $u$ in state $s$ where $s(u)$ is \textit{true}.

\begin{definition}[Classical count-based novelty] \label{def-count-based-novelty}
The count-based novelty \(c^U(s)\) of a newly generated state \(s\) at time-step \(t+1\) given a history of generated states \(s_{0:t}\) and set of variable conjunctions \(U=U^{(k)}\) for some tuple size $k$ is:

\[c^U(s_{t+1}):=\min_{u\in U^+(s_{t+1})}(N^{u}_{t}(s_{t+1}))\]

Where \(N^{u}_t(s_{t+1})\) counts the number of states $s_i \in s_{0:t}$ where $s_i(u)=s_{t+1}(u)$.
\end{definition}

That is, for each tuple $u$ that is \textit{true} in $s_{t+1}$, we count the number of states $s_i \in s_{0:t}$ where $s_i(u)$ is \textit{true}, and we select the minimum out of those counts. 

Following prior work on Novelty \citep{lipovetzky2017best}, we also define a version of count-based novelty which uses \textit{partition functions} to separate the search space into distinct sub-spaces.

\begin{definition}[Partitioned classical count-based novelty] \label{def-part-count-based-novelty}
\textit{The partitioned count-based novelty \(c^U(s)\) of a newly generated state \(s\) at time-step $t+1$ given partition functions \(h_1,...,h_m\) is:}

\[c^U_{h_1,...,h_m}(s_{t+1}):=\min_{u\in U^+(s_{t+1})}(N^{u}_{t;h_1,...,h_m}(s_{t+1}))\]

Where \(N^{u}_{t;h_1,...,h_m}(s_{t+1})\) counts the number of states $s_i \in \{s_{0:t} \mid h_1(s_i)=h_1(s_{t+1})\land...\land h_m(s_i)=h_m(s_{t+1})\}$ where $s_i(u)=s_{t+1}(u)$.
\end{definition}

In other terms, we are obtaining tuple counts relative to the partition of previously generated states where \(h_j(s_i)=h_j(s_{t+1})\) for all \(1\leq j\leq m\), as opposed to the full state history \(s_{0:t}\). It trivially follows that \(N^u_{t;h_1,...,h_m}(s_{t+1})\leq N^u_t(s_{t+1})\) and \(c^U_{h_1,...,h_m}(s_{t+1})\leq c^U(s_{t+1})\).

\subsection{Theoretical results} \label{Theoretical results}

In this section, we justify the notion that classical count-based novelty achieves an efficient exploration of the state space that benefits planner performance, and present the mechanisms through which this is achieved. Firstly, we focus on the exploratory aspect of our heuristic, by detailing how size-1-tuple counts can be leveraged to direct the search towards lesser explored areas of the state space. We do so by exploiting a Hamming distance measure of a state to all previously visited states, as it provides an intuitively appealing means of quantifying how different a newly visited state is to the solver's visitation history. By demonstrating that information on size-1-tuple counts leads to improved bounds with respect to the Hamming distance of newly visited states in Theorems~\ref{bound-1} to \ref{pc-bounds-3}, we highlight the extent to which classical count-based novelty identifies under-explored areas of the state space. This exploratory aspect alone, however, does not validate the heuristic's effectiveness, as it fails to reveal whether the novel information is beneficial to the search. We address this aspect in Theorems~\ref{prob-theorem-1} and \ref{prob-theorem-2} by using information on size-1-tuple counts and average Hamming distance of states to estimate the expected number of novel tuples. \citet{gross2020novel} show that novel tuples benefit search performance by indicating potential new paths towards the goal. Our results identify the two mechanisms through which classical count-based novelty increases the expectation of such novel tuples.

We define node $n_i=n_i(s_i)$ as referring to a state $s_i$, where the sequence $n_{0:t}$ corresponds to sequence $s_{0:t}$. The distinction between a node $n_i$ and its corresponding state $s_i$ lies in the equality operator: $n_i=n_j$ iff $i=j$, implying that $s_i=s_j$, whereas $s_i=s_j$ denotes the equality of all underlying variable values $v$ in $s_i$ and $s_j$. Crucially, throughout the entire section we assume that $U=U^{(1)}=V$, that is, we are only looking at counts over single-variable tuples. We thus simplify the tuple notation by denoting $s^i=s(v_i)$. Let $L=|s|=|V|$, and \textit{Hamming distance} $H(n, n_j)=H(n(s),n_j(s_j))=\sum_{i=0}^{L-1}1_{s^i\neq s_{j}^i}$. We then define \textit{normalized Hamming distance} as $\delta(n, n_j)=\frac{1}{L} (H(n, n_j))=\frac{1}{L}\sum_{i=0}^{L-1}1_{s^i\neq s_{j}^i}$, and the \textit{average normalized Hamming distance} of a node $n$ with respect to all nodes in $n_{0:t}$ as \(\alpha_{0:t}(n) = \frac{1}{W} \sum_{i=0; n_i \neq n}^{t} \delta(n, n_i)\) where $W=t$ if \(n \in n_{0:t}\) or $W=t+1$ otherwise, noting that in the first case we are skipping a node's comparison with itself.

Let the \textit{empirical count distribution} be $\mu_t^i(s)=\frac{N_t^{v_i}(s)}{t+1}$, and $\mu_t^{min}(s)=\min_{i \in V}(\mu_t^i(s))$, noting that the minimum is over the entire set of variables $V=U^{(1)}$ rather than the set of \textit{true} variables $U^{+(1)}$ in a state $s$ used in Definition~\ref{def-count-based-novelty} and \ref{def-part-count-based-novelty}, and $0\leq \mu_t^i(s) \leq 1$. We provide a justification of this change through Propositions~\ref{prop-1} and \ref{prop-2}, demonstrating a correspondence between empirical counts and Hamming distances over $U^{(1)}$ in binary vectors, and the same metrics over the set $U^{+(1)}$ in binary vectors that include negated variables. This allows us to align our results with a STRIPS representation that includes negated variables. For the set of $L$ variables $V$, we define \(s_{\text{neg}}\) for states \(s\) over \(V_{neg}=V \cup \{\neg v \mid v \in V\}\) such that \(s_{\text{neg}}(v) = s(v)\), \(s_{\text{neg}}(\neg v) = \neg s(v)\) $\text{for all } v_i \in V$. Since \(H(s, s') = |\{i \mid s(v_i) \neq s'(v_i)\}|\), we define \(H_{\text{true}}(s, s') = |\{i \mid s(v_i) \neq s'(v_i), s(v_i) = 1\}|\) and \(H_{\text{false}}(s, s') = |\{i \mid s(v_i) \neq s'(v_i), s(v_i) = 0\}|\), noting that \(H(s, s') = H_{\text{true}}(s, s') + H_{\text{false}}(s, s')\).

\begin{proposition} \label{prop-1}
The Hamming distance \(H(s, s')\) between states \(s\) and \(s'\) equals the true Hamming distance \(H_{\text{true}}(s_{\text{neg}}, s'_{\text{neg}})\), considering only variables in \(s_{\text{neg}}\) that are true.
\end{proposition}

\begin{proof}
  Since for each variable $s(v_i)=1 \in s$ there are two variables $s_{neg}(v_i)=1 \in s_{neg}$ and $s_{neg}(\neg v_i)=0$, and for each variable $s(v_i)=0 \in s$ there are two variables $s_{neg}(\neg v_i)=1 \in s_{neg}$ and $s_{neg}(v_i)=0$, follows that \(H_{\text{true}}(s_{neg}, s_{neg}') = H_{\text{true}}(s, s') + H_{\text{false}}(s, s')=H(s, s')\)
\end{proof}

\begin{proposition} \label{prop-2}
The empirical count $N_t^{v_i}(s)$ for any value $s(v_i)$ corresponds to the empirical count $N_{t}^{v_x}(s_{neg})$, where $x=i$ if $s(v_i)=1$ and $x=j$ if $s(v_i)=0$ where $s_{neg}(v_j)\equiv s_{neg}(\neg v_i)$.
\end{proposition}

\begin{proof}
    From the definition of $s_{neg}$, for every variable $s_j(v_i)=0 \in s_{0:t}(v_i)$ we have that $s_{neg;j}(v_i)=0$ and $s_{neg;j}(\neg v_i)=1$. The proof for $s_j(v_i)=1$ case is symmetrical. Proposition~\ref{prop-2} follows.
\end{proof}

It then follows from Proposition~\ref{prop-2} that selecting the minimum count $c^{V_{neg}}(s_{neg;t+1})=\min_{v \in V}(N_{t}^{v}(s_{t+1}))$.

\begin{theorem}\label{bound-1}
The average normalized Hamming distance $\alpha_{0:t}(s)$ of a state $s$ to the $t+1$ states in history $s_{0:t}$ is upper bounded by:
\[\alpha_{0:t}(s)\leq 1-\mu_{t}^{min}(s)\]
\end{theorem}

\begin{proof}
The average Hamming distance of $s(v_i)$ with respect to the value of variable $i$ in all states $s_j \in s_{0:t}$ is equivalent to 
\begin{align*}
    \begin{split}
    \frac{1}{t+1}\sum_{j=0}^{t}1_{s_{j}^i\neq s^i}
    &=1-\frac{1}{t+1}\sum_{j=0}^{t}1_{s_{j}^i=s^i}
    =1-\frac{1}{t+1}N_{t}^{v_i}(s)\\
    &=1-\mu_{t}^i(s)
    \end{split}
\end{align*}

Thus we have
\begin{equation} \label{eq1}
\begin{split}
    \alpha_{0:t}(s)&
    =\frac{1}{t+1}\sum_{j=0}^{t}\frac{1}{L}\sum_{i=0}^{L-1}1_{s_{j}^i\neq s^i}
    =\frac{1}{L}\sum_{i=0}^{L-1}\frac{1}{t+1}\sum_{j=0}^{t}1_{s_{j}^i\neq s^i}\\
    &=\frac{1}{L}\sum_{i=0}^{L-1}(1-\mu_{t}^{i}(s)) \leq \frac{1}{L}\sum_{i=0}^{L-1}(1-\mu_{t}^{min}(s)) \\
    &=1-\mu_{t}^{min}(s)
\end{split}
\end{equation}
\end{proof}

\paragraph{Parent-child average distance comparison} We provide a set of results on the average normalized Hamming distances of a child node with respect to its parent node, and the impact that the count-based novelty of a node has on this value. Incorporating constraints on the changes between parent and child nodes enables us to obtain much tighter bounds compared to Theorem~\ref{bound-1}, reflecting the parent-child dynamic that exists between expanded and generated nodes. Let $n^c$ and $n^p$ be the child and parent node respectively, where an action $a\in A$ is performed on $n^p$ to flip the value of $e$ variables, which we refer to as the \textit{effects}. Let $n^c$ be a newly generated node $n_t$.

\begin{theorem} \label{pc-bounds-1}
\textit{Lower and upper bounds for $\alpha_{0:t}(n^c)$ are given by:}

\[\alpha_{0:t}(n^p)-\frac{t-1}{t}\frac{e}{L} \leq \alpha_{0:t}(n^c) \leq \alpha_{0:t}(n^p)+\frac{t-1}{t}\frac{e}{L}\]
\end{theorem}

\begin{proof}
In the lower bound, all $e$ effect variables change their corresponding valuation to match with all states in history except for the parent node, reducing Hamming distance to each state by 1 for each effect $e$. Parent and child states share all variable valuations except for the $e$ effects, which change valuation from parent to child node. This yields, for all cases where $n'\in n_{0:t}$, $n'\neq n^c$ and $n'\neq n^p$
\begin{equation} \label{eq2-1}
    \delta(n^c, n') \geq \frac{1}{L} (H(n^p, n') - e) = \delta(n^p, n') - \frac{e}{L}
\end{equation} 
\begin{equation} \label{eq2-2}
    \delta(n^c, n^p)=\frac{1}{L} (H(n^p, n^p) + e) = \frac{e}{L}
\end{equation}

We can redefine the average $\alpha_{0:t}(n^c)$ as 
\begin{equation} \label{eq2-3}
    \alpha_{0:t}(n^c)=\frac{1}{t}\bigg[\sum_{i=0;n_i \notin \{n^p,n^c\}}^t\Big(\delta(n^c, n_i)\Big) + \delta(n^c,n^p)\bigg]
\end{equation}
Since we define that $n^c=n_{t}$:
\begin{equation} \label{eq2-4}
    \alpha_{0:t}(n^p)=\frac{1}{t}\bigg[\sum_{i=0;n_i \notin n^p}^{t-1}\Big(\delta(n^p,n_i)\Big)+\delta(n^c,n^p)\bigg]
\end{equation}
Substituting (\ref{eq2-1}) and (\ref{eq2-2}) into (\ref{eq2-3}), noting that $\sum_{i=0;n_i \notin \{n^p,n^c\}}^t(\frac{e}{L})=(t-1)\frac{e}{L}$, and then substituting (\ref{eq2-4}) yields Theorem~\ref{pc-bounds-1}.

For the upper bound, we note that it is symmetrical in that in the upper bound all effects \textit{e} are novel, that is, their variable valuation in $n^c$ has never been observed in $n_{0:t-1}$. Thus we get $\delta(n^c,n')\leq\delta(n^p,n')+\frac{e}{L}$. Following the same procedure yields the upper bound.
\end{proof}

\begin{theorem}\label{pc-bounds-2}
\textit{Given a minimum empirical count distribution $\mu = \mu_{t-1}^{min}(n^c)$, the upper bound $\alpha_{0:t}(n^c)$ with respect to $\mu$ is given by:}
\[
\alpha_{0:t}(n^c)\leq\alpha_{0:t}(n^p)+\frac{t-1}{t}\frac{e-2e\mu}{L}
\]
\end{theorem}

\begin{proof}
Since $\mu$ is the minimum feature occurrence, acting as a constraint, upper bound occurs when all effects $e$ have occurrence equal to $\mu$, that is, the minimum possible occurrence they are allowed to have. Thus, for $t-1$ nodes $n'\in n_{0:t-1}$, $n'\neq n^p$, we have that $\delta(n^c, n')=\frac{1}{L} (H(n^p, n') + e)$ a total of $(1-\mu)\cdot (t-1)$ times, and $\delta(n^c, n')=\frac{1}{L} (H(n^p, n') - e)$ a total of $\mu\cdot (t-1)$ times.
Proof follows from derivation in Theorem~\ref{pc-bounds-1}.
\end{proof}

\begin{theorem}\label{pc-bounds-3}
\textit{Lower bound for $\alpha_{0:t}(n^c)$ when $\mu_{t-1}^{min}(n^c)=0$ is given by:}
\[
\alpha_{0:t}(n^c) \geq \alpha_{0:t}(n^p)-\frac{t-1}{t}\frac{e-2}{L}
\]
\begin{proof}
In the lower bound, one effect is novel, and $e-1$ effects match all previous history except $n^p$. Thus \(\delta(n^c, n') = \frac{1}{L} (H(n^p, n') - (e-1) + 1)\). Proof follows from derivation in Theorem~\ref{pc-bounds-1}.
\end{proof}
\end{theorem}

A comparison of the bounds in Theorem~\ref{pc-bounds-1} with those in Theorems~\ref{pc-bounds-2} and \ref{pc-bounds-3} demonstrates the relation between novelty count and Hamming distance through improved bounds, in terms of changes in the average distances from parent to child node, for nodes with a low empirical count $N$, which acts through the empirical count distribution $\mu$. The upper bound in Theorem~\ref{pc-bounds-2} details the main improvement, signalling greater potential of the child node being located in newer areas of the state space. Theorem~\ref{pc-bounds-3} is a notable special case for novel variable valuations never encountered before, which guarantees an improvement of the lower bound through the novel information that could not have already been observed in the parent node. Through the recursive nature of Theorems~\ref{pc-bounds-1} to \ref{pc-bounds-3}, we also conclude that paths consisting of low count-based novelty nodes are more likely to exhibit rapidly increasing average Hamming distances, thus facilitating a quicker exploration of novel state spaces. We cannot establish a tighter lower bound in Theorem~\ref{pc-bounds-2} because the least common feature might not be an effect, however modifying count-based novelty metrics to consider effect occurrences could overcome this limitation. Still, greater Hamming distances alone fail to explain how count-based novelty benefits search efficiency. We provide Theorems~\ref{prob-theorem-1} and \ref{prob-theorem-2} to tie our results to prior theoretical contributions on novelty-based search (see \citep{dold2024novelty,gross2020novel,lipovetzky2014width}) through an analysis of the expected count of novel tuples of size \textit{k} (\textit{k}-tuples).

\paragraph{Estimating novel \textit{k}-tuples}  
Let history $s_{0:t}$ represent $t+1$ independent and uniformly distributed binary vectors of size $L$. A tuple is novel if its valuation in $s = s_{t+1}$ was not observed in any state in history $s_{0:t}$.

\begin{theorem} \label{prob-theorem-1}
The expected number of novel tuples of size $k$ found in $s=s_{t+1}$ given search history $s_{0:t}$ is given by:
\begin{equation} \label{prob1}
    \binom{L}{k}\Big[1-(1-\alpha_{0:t}(s))^k\Big]^{t+1}
\end{equation}
\end{theorem}

\begin{proof}
    From equation (\ref{eq1}) we can obtain $\alpha_{0:t}(s)=\frac{1}{(t+1)\cdot L}\sum_{j=0}^{t}\sum_{i=0}^{L-1}1_{s_{j}^i\neq s_{}^i}=\mathbb{E}_{s_j \in s_{0:t}, v_i \in V}[\mathbf{1}_{\{s_j(v_i) \neq s(v_i)\}}]=P(s_j(v_i) \neq s(v_i) \text{ for some } j, i)$. Thus, the probability that it has the same value becomes $1-\alpha_{0:t}(s)$, and for a tuple of size $k$, the probability that any of its constituent variable values is different in $s_j$ than in $s$ is $1-(1-\alpha_{0:t}(s))^k$. Calculating the union for a tuple over the full history and multiplying by the number of possible tuples of size $k$ yields the expectation in (\ref{prob1}).
\end{proof}

\begin{theorem} \label{prob-theorem-2}
 The expected number of novel tuples of size $k$ found in state $s=s_{t+1}$ given information on occurrence count $N=N_t^v(s)$ for some variable $v\in V$ and search history $s_{0:t}$ is given by:
\[
   \binom{L-1}{k}\Big[1-(1-\beta_{0:t}(s))^k\Big]^{t+1} + \binom{L-1}{k-1}\Big[1-(1-\beta_{0:t}(s))^{k-1}\Big]^N
\]
where $\beta_{0:t}(s)$ represents the average normalized Hamming distance after discounting the contribution of variable $v$:
\[
    \beta_{0:t}(s)=\frac{\alpha_{0:t}(s)\cdot L-(1-\frac{N}{t+1})}{L-1}
\]
\end{theorem}

\begin{proof}
The left-hand side component of the addition is given by equation (\ref{prob1}) taken over tuples deriving from variables except for the variable $v$ whose empirical count $N$ we observe. The right-hand side component is given by the probability $1-(1-\beta_{0:t}(s))^{k-1}$ that, for some variable $x$ other than $v$ in a \textit{k}-tuple containing $v$ and in a state $s_j$ where $s_j(v)=s(v)$, $s_j(x)\neq s(x)$. Thus, the tuple's valuation in $s_j$ is different than in $s$.
Taking a union over $N$ states with matching $v$ valuation and multiplying by the total number of tuples in $s$ containing $v$ yields the right-hand side component. Summing the two expectations proves the theorem.
\end{proof}

\begin{figure}[tbh]
  \centering
  \begin{subfigure}{.472\columnwidth}  
    \includegraphics[width=\linewidth]{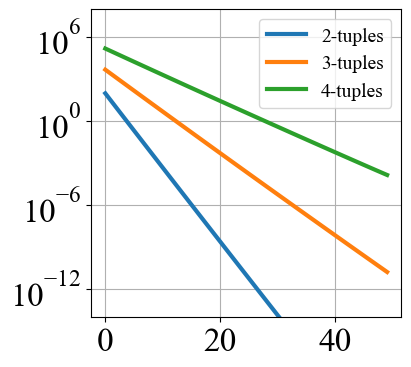}
    \caption{$\mathbb{E}[\text{\# novel \textit{k}-tuple}]$ vs. $N$}
    \label{fig:sub1}
  \end{subfigure}\hfill 
  \begin{subfigure}{.48\columnwidth}
    \includegraphics[width=\linewidth]{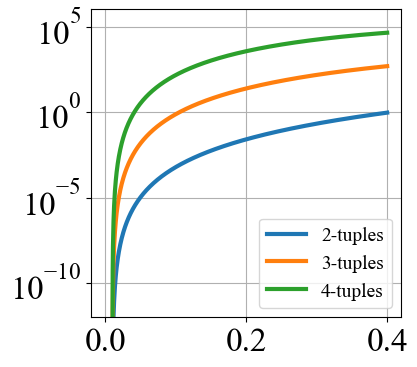}
    \caption{$\mathbb{E}[\text{\# novel \textit{k}-tuple}]$ vs. $\alpha(s)$}
    \label{fig:sub2}
  \end{subfigure}
  \vspace{5mm}
  \caption{$\mathbb{E}[\text{\# novel \textit{k}-tuple}]$ according to Theorem \ref{prob-theorem-2}. In (a) $N$ is a variable, and in (b) $\alpha(s)$ is a variable. Otherwise, parameters are set as $L=100$, $t=50000$, $N=5$, $\alpha(s)=0.3$. A realistic $\alpha(s)$ value was determined through simulation$^1$.
  }
  \label{fig:count_prob_analysis}
  \vspace{8mm}
\end{figure}

Theorem~\ref{prob-theorem-1} reveals that, without count information, the expectation decreases exponentially with increasing $t$, rendering the measure effective only for small history sizes. Conversely, Theorem~\ref{prob-theorem-2} introduces a component independent of $t$ and exponential in count number $N$, emphasizing the crucial role of the minimum count function in identifying states likely to contain novel tuples, necessary to fulfill new action preconditions.

We thus highlight the double role played by count-based novelty in inducing novel tuples, as shown in Figure~\ref{fig:count_prob_analysis}: directly, through the greater probability that a novel tuple may contain a variable with low occurrence $N$, as well as indirectly, through the effect that a low count $N$ induces a greater average Hamming distance compared to history $s_{0:t}$ (Theorems \ref{bound-1} to \ref{pc-bounds-3}), which in turn increases the expected number of novel tuples. We note that this analysis is na\"ive in that it does not account for the structure present in domains, which alters the probability of co-occurrence of variables. We seek to emphasize that based on our proposed theoretical analysis, it is not the greater Hamming distance alone that leads to meaningful beneficial performance, as shown by Theorem 7, but rather its effect in conjunction with low variable occurrence in the second term of Theorem 8.

\footnotetext[1]{We remove a randomly generated root binary vector of size $L$=100 from the front of an open list, generate 4 children nodes uniformly at random flipping 3 binary variables each, and append each into the open list. We repeated the process to create 10000 nodes, and measured the average Hamming distance of the last 100 nodes, yielding an average value of $\approx 35$.}

\section{Trimmed Open List} \label{Trimmed Open List}
Balancing the amount of memory occupied by low-rank nodes is a common strategy which allows for better ranked exploratory nodes to appear further down the search. Polynomial width-novelty planners prune nodes with novelty value greater than a threshold, as they are deemed not useful for the search. Similarly, count exploration methods generate many nodes with high counts, which are unlikely to ever be expanded. However, adopting a threshold as in the width-novelty case is unfeasible due to the granularity of the metric.

To address the challenge of high memory usage by poorly ranked nodes in the open list, we introduce the \textit{Trimmed Open List} (Alg. \ref{alg:cap}). Built on a binary heap, this open list limits its growth by pruning less promising nodes when it exceeds a predefined size limit $Z$. This pruning process involves randomly selecting a leaf node, comparing its heuristic value with a new node $n$ using the open list's comparison function, and then pruning or swapping nodes based on their heuristic values. A unique heapify-up operation is applied to the inserted leaf, which, unlike standard heaps, is not required to be the last element.

Furthermore, we developed a \textit{Double Trimmed Open List} for heuristic alternation \citep{roger2010more}, accommodating dual open lists for node insertion under distinct heuristics and enabling alternate node retrieval. This variant employs the same pruning strategy but distinguishes itself by tracking each node's interaction with the open lists $-$ either being popped or trimmed. A node becomes eligible for deletion when its interaction count equals the number of lists it is associated with, provided it is not in the closed list. This ensures a node is removed only when it is confirmed to be redundant, safeguarding against premature deletion crucial for the lazy expansion of successors.

\section{Experiments} \label{Experiments}

Our experiments were conducted using Downward Lab's experiment module \citep{seipp2017downward}, adhering to the IPC satisficing track constraints of 1800 seconds and 8 \textit{GB} memory. Each test was ran on a single core of a cloud instance AMD EPYC 7702 2GHz processor. We implemented all proposed planners in C++, using LAPKT's \citep{lapkt} planning modules. For hybrid experiments, LAMA-First \citep{richter2010lama} and Scorpion-Maidu \citep{correa-et-al-ipc2023c,seipp2020saturated} served as backend components, employing Fast-Downward \citep{helmert2006fast} and the IPC2023 code repository \citep{maidu2023ipc}, respectively. Except for Approximate-BFWS, BFWS variants utilized the FD-grounder for grounding \citep{helmert2009concise}, however in problems where the FD grounder produces axioms (unsupported by LAPKT), LAPKT automatically switched to the Tarski grounder \citep{tarski:github:18}. Approximate-BFWS exclusively used the Tarski grounder, following its initial setup and IPC-2023 configuration. We utilized IPC satisficing track benchmarks as in \citep{singh2021approximate}, selecting the latest problem sets for recurring domains. We conducted two sets of experiments. The first benchmarked our planners against the base $\text{BFWS}(f_5)$ solver, evaluating the degree to which our proposed classical count-based novelty and trimmed open list techniques improve the coverage of $\text{BFWS}(f_5)$ and its exploration efficiency, measured as the number of expansions required to find a solution. The second set of experiments compared our hybrid configurations to Dual-BFWS, Approximate-BFWS, LAMA-First, and a "first" version of the IPC-2023 satisficing track winner Scorpion-Maidu that runs its first iteration, in order to assess the coverage gains obtainable by adopting our proposed frontend solver alongside existing solvers in a dual configuration, relying on memory thresholds alongside more traditional time thresholds to trigger the frontend to fallback.

\begin{algorithm}[t]
\caption{Trimmed Open List}\label{alg:cap}
\begin{algorithmic}

\Procedure{Trimmed Open List}{new node $N$, heap $H$, heap size limit $Z$}
\State $S \gets \text{size of } H$
\If{$S < Z$} \Comment{If heap hasn't reached size limit $Z$}
    \State $\text{insert } N \text{ into } H$
    \State $\textit{heapify-up}(H)$ \Comment{Reorder last element}
\Else
    \State $i\gets \text{uniformly random leaf index of } H$
    \State $O\gets H[i]$ \Comment{Node at random leaf index}
    \If{$N \text{ has a better heuristic value than } O$}
        \State $H[i] \gets N$ \Comment{Replace $O$ with $N$}
        \State $\textit{heapify-up}(H, i)$ \Comment{Reorder element at index $i$}
        \State $\text{discard } O$
    \Else
        \State $\text{discard } N$
    \EndIf
\EndIf
\EndProcedure

\end{algorithmic}
\end{algorithm}

\subsection{Count-based solvers}
We define three new planning solvers to evaluate the performance of our proposed trimmed open list and classical partitioned count-based novelty techniques. All our solvers are based on the $\text{BFWS}(f_5)$ search algorithm \citep{lipovetzky2017best}. $f_5$ is the evaluation function $\langle w,\#g \rangle$ where $w$ is the novelty measure and the goal counter $\#g$ counts the number of atomic goals not true in $s$. The novelty measure $w$ is computed given partition functions $\#g$ and $\#r(s)$, that is $w_{\langle\#g,\#r\rangle}$ (see \cite{lipovetzky2017best}). We use $w$ to refer to both width-novelty as well our proposed count-based novelty metric, adopting notation $f_5(X)$, where $X$ is the chosen novelty measure $w$. Let $W_x$ be the partitioned width-novelty metric with $\text{max-width}=x$ \citep{lipovetzky2017best}. Let $C_1=c^{V}$ be a partitioned classical count-based novelty metric over size-1 features $v\in V$. 
We define the following search algorithms:
\begin{itemize}
\item $BFWS_t(f_5(W_2))$: Standard BFWS($f_5$) with max-width=2. Subscript \textit{t} denotes use of a Single Trimmed Open List.
\item $BFCS_t(f_5(C_1))$: A BFWS($f_5$) solver where $w=C_1$. Subscript \textit{t} denotes use of a Single Trimmed Open List.
\item  $BFNoS_t(f_5(C_1),f_5(W_2))$: Best-first search solver using a Double Trimmed Open List, employing the evaluation function $f_5(C_1)$ $-$ with $w=C_1$ $-$ for first open list and evaluation function $f_5(W_2)$ $-$ with $w=W_2$ $-$ for the second open list, alternating expansions between lists. We refer to this solver as BFNoS (Best First Novelty Search), as it uses multiple novelty heuristics.
\end{itemize}
The trimmed open list is capped at a constant depth $D = 18$ (maximum size $Z = 524,287$) determined by empirical testing. We note that empirically, small changes in depth ($D=17$ or $D=19$) did not alter coverage results beyond the deviation recorded in experiments.

\begin{table}[htb]
  \centering
  \resizebox{\columnwidth}{!}{
  \begin{tabular}{ccccc}
    \hline
    \rule{0pt}{2.0ex}& & \multicolumn{3}{c}{\textbf{Coverage}} \\
    \cline{3-5}
    \rule{0pt}{2.6ex}\textbf{Model} & \textbf{\% Score} & \textbf{Total (1831)} & \textbf{IPC 2023 (140)} & \textbf{IPC 2018 (200)}\\ 
    \hline
    \rule{0pt}{2.6ex}$BFWS(f_5(W_2))$ &76.76\%& 1510 & 67 & 120\\
    $BFCS(f_5(C_1))$ &77.57\%& 1510 & 75 & 129\\
    $BFWS_t(f_5(W_2))$ & \makecell{79.78\%$\pm$0.13} & \makecell{1555$\pm$1.64} & \makecell{66$\pm$0.45} & \makecell{134$\pm$1.82} \\
    $BFCS_t(f_5(C_1))$ & \makecell{81.53\%$\pm$0.33} & \makecell{1568$\pm$4.76} & \makecell{78$\pm$0.89} & \makecell{146$\pm$2.79} \\
    $BFNoS$ & \makecell{\textbf{83.32\%$\pm$0.18}}& \makecell{\textbf{1600$\pm$3.90}}& \makecell{\textbf{87$\pm$1.10}}& \makecell{\textbf{149$\pm$1.34}}\\
    \hline
  \end{tabular}}
  \vspace{1mm}
  \caption{
  \% score and coverage comparison of proposed variants. \% score is the average of the \% of instances solved in each individual domain, calculated over all benchmark domains. Values represent the mean and include the standard deviation across 5 measurements.
  }
  \label{tab:experiment-1-table}
\end{table}

\paragraph{Analysis of proposed techniques}
We refer to the results of our experiment in Table \ref{tab:experiment-1-table} showing significant improvement in $\text{BFWS}(f_5)$'s coverage. 
The trimmed open list's smaller memory footprint substantially boosts the instance coverage of BFCS and BFWS solvers alike, demonstrating the versatility of its node filtering mechanism even when dealing with the $W_2$ heuristic's narrower range.

$\text{BFCS}_t(f_5)$ outperforms $\text{BFWS}_t(f_5)$ in both coverage and normalized score. This advantage is especially evident in problems with high atomic widths, such as \textit{Ricochet-Robots} \citep{ricochetrobots} from IPC2023 \citep{ipc2023}, where $\text{BFCS}_t(f_5)$ consistently solves 19 out of 20 instances compared to the $\text{BFWS}_t(f_5)$'s single solve. This underscores count-based novelty's scalability in complex problems, contrasting with $W_x$ metrics which have to revert to secondary heuristics after exhausting novel nodes, and demonstrates the \textit{O(n)} count-based novelty variant's capacity to seek novel tuples of size $>1$ as predicted by our analysis in Section \ref{Theoretical results}. However, while $C_1$ can prioritize states with a higher expected number of 2-tuples, it cannot explicitly detect the presence of 2-tuples like $W_2$, and $\text{BFCS}_t(f_5)$ does show reduced coverage compared to $\text{BFWS}_t(f_5)$ in various domains, suggesting that $W_2$ and $C_1$ heuristics offer complementary strengths. 

This synergy is exemplified by $\text{BFNoS}_t(f_5(C_1),f_5(W_2))$, which surpasses both in coverage due to its dual-heuristic approach. Notably, it also secures a significant 3.5\% gain in normalized scores compared to $\text{BFWS}_t(f_5)$, indicative of the planner's enhanced cross-domain generalization. To our knowledge, this is the first instance demonstrating performance gains from combining distinct goal-unaware exploration heuristics in planning, as opposed to combining goal-aware exploitation heuristics as in \citep{roger2010more}. On problems solved by both BFNoS and $\text{BFWS}_t(f_5)$, integrating the $C_1$ heuristic also reduces the number of node expansions required on average to solve instances as their size increases. This is illustrated in the upper-right quadrant of Figure~\ref{fig:expansion_comparison}, which highlights a general improvement in tackling large domains.

\begin{figure}[tbh]
    \centering
    \begin{overpic}[width=0.6\linewidth]{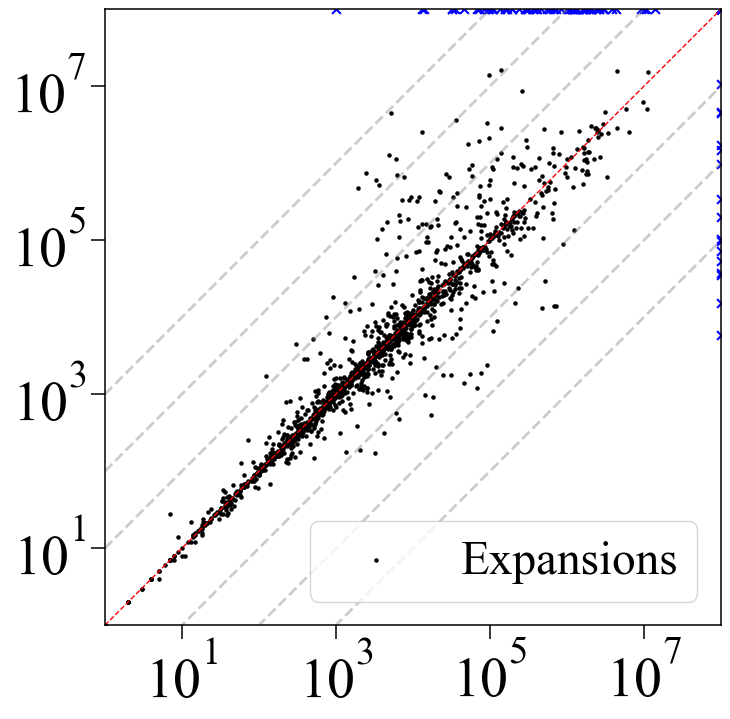}
        \put(55,-2){\makebox(0,0){BFNoS}} 
        \put(-3,55){\rotatebox{90}{\makebox(0,0){BFWS$_t(f_5)$}}} 
    \end{overpic}
    \vspace{3mm} 
    \caption{Number of nodes expanded across instances solved by $\text{BFNoS}$ and $\text{BFWS}_t(f_5)$. Blue crosses represent instances not solved by at least one planner.}
    \label{fig:expansion_comparison}
\end{figure}

\begin{table*}[t]
    \centering
    \fontsize{8}{8}\selectfont
    \begin{tabular}{|l|c|c|c|c|c|c|c|c|c|} \hline
        \rule{0pt}{8pt}Domain & BFNoS & Dual- & Apx-BFWS &  LAMA- & Maidu &  Maidu &  BFNoS- & BFNoS- & BFNoS- \\
        & & BFWS & (Tarski) & First &  & with $h^2$ & Dual & LAMA & Maidu-${h^2}$\\ \hline
        \rule{0pt}{8pt}agricola-sat18-strips (20)&\makecell{15$\pm$0.0}&\makecell{12}&\makecell{\textbf{18$\pm$0.9}}&\makecell{12}&\makecell{12}&\makecell{13}&\makecell{15$\pm$0.5}&\makecell{15$\pm$0.5}&\makecell{15$\pm$0.5} \\ 
        airport (50)&\makecell{\textbf{47$\pm$0.6}}&\makecell{46}&\makecell{\textbf{47$\pm$0.6}}&\makecell{34}&\makecell{38}&\makecell{45}&\makecell{46$\pm$0.6}&\makecell{46$\pm$0.5}&\makecell{46$\pm$0.6} \\
        caldera-sat18-adl (20)&\makecell{18$\pm$0.0}&\makecell{\textbf{19}}&\makecell{\textbf{19$\pm$0.6}}&\makecell{16}&\makecell{16}&\makecell{16}&\makecell{16$\pm$0.0}&\makecell{17$\pm$0.5}&\makecell{18$\pm$0.0} \\ 
        cavediving-14-adl (20)&\makecell{\textbf{8$\pm$0.5}}&\makecell{\textbf{8}}&\makecell{\textbf{8$\pm$0.5}}&\makecell{7}&\makecell{7}&\makecell{7}&\makecell{\textbf{8$\pm$0.0}}&\makecell{\textbf{8$\pm$0.0}}&\makecell{\textbf{8$\pm$0.5}} \\
        childsnack-sat14-strips (20)&\makecell{1$\pm$1.1}&\makecell{\textbf{9}}&\makecell{5$\pm$0.6}&\makecell{6}&\makecell{6}&\makecell{6}&\makecell{8$\pm$0.0}&\makecell{6$\pm$0.0}&\makecell{6$\pm$0.5} \\ 
        citycar-sat14-adl (20)&\makecell{\textbf{20$\pm$0.0}}&\makecell{\textbf{20}}&\makecell{\textbf{20$\pm$0.0}}&\makecell{5}&\makecell{6}&\makecell{6}&\makecell{\textbf{20$\pm$0.0}}&\makecell{\textbf{20$\pm$0.0}}&\makecell{\textbf{20$\pm$0.0}} \\ 
        data-network-sat18-strips (20)&\makecell{17$\pm$0.6}&\makecell{13}&\makecell{\textbf{19$\pm$0.5}}&\makecell{13}&\makecell{16}&\makecell{16}&\makecell{16$\pm$0.8}&\makecell{15$\pm$1.1}&\makecell{16$\pm$0.8} \\ 
        depot (22)&\makecell{\textbf{22$\pm$0.0}}&\makecell{\textbf{22}}&\makecell{\textbf{22$\pm$0.0}}&\makecell{20}&\makecell{\textbf{22}}&\makecell{\textbf{22}}&\makecell{\textbf{22$\pm$0.0}}&\makecell{\textbf{22$\pm$0.0}}&\makecell{\textbf{22$\pm$0.0}} \\ 
        flashfill-sat18-adl (20)&\makecell{14$\pm$1.3}&\makecell{\textbf{17}}&\makecell{15$\pm$1.6}&\makecell{14}&\makecell{15}&\makecell{14}&\makecell{\textbf{17$\pm$0.5}}&\makecell{16$\pm$0.6}&\makecell{16$\pm$0.9} \\
        floortile-sat14-strips (20)&\makecell{2$\pm$0.5}&\makecell{2}&\makecell{2$\pm$0.0}&\makecell{2}&\makecell{2}&\makecell{\textbf{20}}&\makecell{2$\pm$0.0}&\makecell{2$\pm$0.0}&\makecell{\textbf{20$\pm$0.0}} \\ 
        folding (20)&\makecell{9$\pm$0.0}&\makecell{5}&\makecell{5$\pm$0.5}&\makecell{\textbf{11}}&\makecell{\textbf{11}}&\makecell{\textbf{11}}&\makecell{9$\pm$0.0}&\makecell{9$\pm$0.0}&\makecell{9$\pm$0.0} \\
        freecell (80)&\makecell{\textbf{80$\pm$0.0}}&\makecell{\textbf{80}}&\makecell{\textbf{80$\pm$0.0}}&\makecell{79}&\makecell{\textbf{80}}&\makecell{\textbf{80}}&\makecell{\textbf{80$\pm$0.0}}&\makecell{\textbf{80$\pm$0.0}}&\makecell{\textbf{80$\pm$0.0}} \\ 
        hiking-sat14-strips (20)&\makecell{\textbf{20$\pm$0.0}}&\makecell{18}&\makecell{\textbf{20$\pm$0.0}}&\makecell{\textbf{20}}&\makecell{\textbf{20}}&\makecell{\textbf{20}}&\makecell{\textbf{20$\pm$0.0}}&\makecell{\textbf{20$\pm$0.0}}&\makecell{\textbf{20$\pm$0.0}} \\ 
        labyrinth (20)&\makecell{15$\pm$0.5}&\makecell{5}&\makecell{\textbf{18$\pm$0.5}}&\makecell{1}&\makecell{0}&\makecell{2}&\makecell{15$\pm$0.5}&\makecell{15$\pm$0.5}&\makecell{15$\pm$0.5} \\ 
        maintenance-sat14-adl (20)&\makecell{\textbf{17$\pm$0.0}}&\makecell{\textbf{17}}&\makecell{\textbf{17$\pm$0.0}}&\makecell{11}&\makecell{13}&\makecell{13}&\makecell{\textbf{17$\pm$0.0}}&\makecell{\textbf{17$\pm$0.0}}&\makecell{\textbf{17$\pm$0.0}} \\ 
        mystery (30)&\makecell{18$\pm$0.6}&\makecell{\textbf{19}}&\makecell{\textbf{19$\pm$0.0}}&\makecell{\textbf{19}}&\makecell{\textbf{19}}&\makecell{\textbf{19}}&\makecell{\textbf{19$\pm$0.0}}&\makecell{\textbf{19$\pm$0.0}}&\makecell{\textbf{19$\pm$0.0}} \\ 
        nomystery-sat11-strips (20)&\makecell{14$\pm$0.8}&\makecell{\textbf{19}}&\makecell{14$\pm$0.5}&\makecell{11}&\makecell{\textbf{19}}&\makecell{18}&\makecell{\textbf{19$\pm$0.0}}&\makecell{15$\pm$0.6}&\makecell{17$\pm$0.0} \\ 
        nurikabe-sat18-adl (20)&\makecell{16$\pm$0.6}&\makecell{14}&\makecell{17$\pm$0.5}&\makecell{9}&\makecell{11}&\makecell{16}&\makecell{17$\pm$0.6}&\makecell{17$\pm$0.0}&\makecell{\textbf{18$\pm$0.0}} \\  
        org-synth-split-sat18-strips (20)&\makecell{8$\pm$0.5}&\makecell{12}&\makecell{8$\pm$0.0}&\makecell{\textbf{14}}&\makecell{\textbf{14}}&\makecell{\textbf{14}}&\makecell{11$\pm$0.5}&\makecell{\textbf{14$\pm$0.0}}&\makecell{\textbf{14$\pm$0.9}} \\ 
        parcprinter-sat11-strips (20)&\makecell{9$\pm$0.6}&\makecell{16}&\makecell{11$\pm$1.3}&\makecell{\textbf{20}}&\makecell{\textbf{20}}&\makecell{\textbf{20}}&\makecell{\textbf{20$\pm$0.0}}&\makecell{\textbf{20$\pm$0.0}}&\makecell{\textbf{20$\pm$0.0}} \\ 
        pathways (30)&\makecell{26$\pm$0.9}&\makecell{\textbf{30}}&\makecell{28$\pm$1.1}&\makecell{23}&\makecell{25}&\makecell{25}&\makecell{\textbf{30$\pm$0.0}}&\makecell{27$\pm$0.8}&\makecell{27$\pm$0.7} \\
        pipesworld-notankage (50)&\makecell{\textbf{50$\pm$0.0}}&\makecell{\textbf{50}}&\makecell{\textbf{50$\pm$0.0}}&\makecell{43}&\makecell{45}&\makecell{45}&\makecell{\textbf{50$\pm$0.0}}&\makecell{\textbf{50$\pm$0.0}}&\makecell{\textbf{50$\pm$0.0}} \\ 
        pipesworld-tankage (50)&\makecell{43$\pm$1.6}&\makecell{42}&\makecell{\textbf{44$\pm$0.6}}&\makecell{43}&\makecell{43}&\makecell{43}&\makecell{43$\pm$0.8}&\makecell{43$\pm$0.5}&\makecell{43$\pm$0.6} \\ 
        recharging-robots (20)&\makecell{\textbf{14$\pm$0.6}}&\makecell{12}&\makecell{\textbf{14$\pm$0.8}}&\makecell{13}&\makecell{13}&\makecell{13}&\makecell{\textbf{14$\pm$0.5}}&\makecell{\textbf{14$\pm$0.0}}&\makecell{\textbf{14$\pm$0.5}} \\ 
        ricochet-robots (20)&\makecell{\textbf{20$\pm$0.5}}&\makecell{\textbf{20}}&\makecell{18$\pm$0.6}&\makecell{14}&\makecell{18}&\makecell{18}&\makecell{\textbf{20$\pm$0.0}}&\makecell{\textbf{20$\pm$0.0}}&\makecell{\textbf{20$\pm$0.0}} \\ 
        rubiks-cube (20)&\makecell{5$\pm$0.0}&\makecell{6}& \makecell{5$\pm$0.6}&\makecell{\textbf{20}}&\makecell{\textbf{20}}&\makecell{\textbf{20}}&\makecell{5$\pm$0.0}&\makecell{\textbf{20$\pm$0.0}}&\makecell{16$\pm$0.6} \\ 
        satellite (36)&\makecell{34$\pm$0.8}&\makecell{33}& \makecell{34$\pm$0.5}&\makecell{\textbf{36}}&\makecell{\textbf{36}}&\makecell{\textbf{36}}&\makecell{34$\pm$0.6}&\makecell{35$\pm$0.0}&\makecell{35$\pm$0.0} \\ 
        schedule (150)&\makecell{149$\pm$1.3}&\makecell{\textbf{150}}&  \makecell{149$\pm$1.3}&\makecell{\textbf{150}}&\makecell{\textbf{150}}&\makecell{\textbf{150}}&\makecell{149$\pm$0.7}&\makecell{\textbf{150$\pm$0.0}}&\makecell{\textbf{150$\pm$0.0}} \\ 
        settlers-sat18-adl (20)&\makecell{13$\pm$1.5}&\makecell{7}& \makecell{12$\pm$0.7}&\makecell{17}&\makecell{\textbf{18}}&\makecell{\textbf{18}}&\makecell{12$\pm$0.5}&\makecell{17$\pm$0.0}&\makecell{17$\pm$0.5} \\ 
        slitherlink (20)&\makecell{\textbf{5$\pm$0.6}}&\makecell{\textbf{5}}& \makecell{\textbf{5$\pm$0.7}}&\makecell{0}&\makecell{0}&\makecell{0}&\makecell{\textbf{5$\pm$0.5}}&\makecell{3$\pm$0.6}&\makecell{4$\pm$0.7} \\ 
        snake-sat18-strips (20)&\makecell{\textbf{20$\pm$0.0}}&\makecell{17}& \makecell{\textbf{20$\pm$0.0}}&\makecell{5}&\makecell{14}&\makecell{14}&\makecell{\textbf{20$\pm$0.0}}&\makecell{\textbf{20$\pm$0.0}}&\makecell{\textbf{20$\pm$0.0}} \\ 
        sokoban-sat11-strips (20)&\makecell{15$\pm$1.1}&\makecell{17}& \makecell{14$\pm$0.9}&\makecell{19}&\makecell{19}&\makecell{\textbf{20}}&\makecell{15$\pm$0.5}&\makecell{19$\pm$0.0}&\makecell{\textbf{20$\pm$0.0}} \\ 
        spider-sat18-strips (20)&\makecell{17$\pm$1.3}&\makecell{16}& \makecell{16$\pm$1.1}&\makecell{16}&\makecell{16}&\makecell{17}&\makecell{\textbf{18$\pm$0.0}}&\makecell{\textbf{18$\pm$0.0}}&\makecell{\textbf{18$\pm$0.9}} \\ 
        storage (30)&\makecell{\textbf{30$\pm$0.5}}&\makecell{29}& \makecell{\textbf{30$\pm$0.0}}&\makecell{20}&\makecell{25}&\makecell{25}&\makecell{29$\pm$0.5}&\makecell{29$\pm$0.0}&\makecell{29$\pm$0.6} \\ 
        termes-sat18-strips (20)&\makecell{10$\pm$0.8}&\makecell{10}& \makecell{5$\pm$1.5}&\makecell{\textbf{16}}&\makecell{14}&\makecell{14}&\makecell{10$\pm$0.5}&\makecell{14$\pm$0.0}&\makecell{14$\pm$0.0} \\ 
        tetris-sat14-strips (20)&\makecell{\textbf{20$\pm$0.0}}&\makecell{17}& \makecell{\textbf{20$\pm$0.0}}&\makecell{16}&\makecell{17}&\makecell{20}&\makecell{\textbf{20$\pm$0.0}}&\makecell{\textbf{20$\pm$0.0}}&\makecell{\textbf{20$\pm$0.0}} \\ 
        thoughtful-sat14-strips (20)&\makecell{\textbf{20$\pm$0.0}}&\makecell{\textbf{20}}& \makecell{\textbf{20$\pm$0.2}}&\makecell{15}&\makecell{19}&\makecell{19}&\makecell{\textbf{20$\pm$0.0}}&\makecell{\textbf{20$\pm$0.0}}&\makecell{\textbf{20$\pm$0.0}} \\ 
        tidybot-sat11-strips (20)&\makecell{\textbf{20$\pm$0.0}}&\makecell{18}& \makecell{\textbf{20$\pm$0.2}}&\makecell{17}&\makecell{\textbf{20}}&\makecell{\textbf{20}}&\makecell{\textbf{20$\pm$0.0}}&\makecell{\textbf{20$\pm$0.0}}&\makecell{\textbf{20$\pm$0.0}} \\ 
        transport-sat14-strips (20)&\makecell{\textbf{20$\pm$0.0}}&\makecell{\textbf{20}}& \makecell{\textbf{20$\pm$0.2}}&\makecell{17}&\makecell{18}&\makecell{16}&\makecell{\textbf{20$\pm$0.5}}&\makecell{\textbf{20$\pm$0.0}}&\makecell{\textbf{20$\pm$0.0}} \\
        trucks-strips (30)&\makecell{8$\pm$0.8}&\makecell{19}& \makecell{13$\pm$1.5}&\makecell{18}&\makecell{20}&\makecell{\textbf{22}}&\makecell{17$\pm$0.5}&\makecell{16$\pm$0.0}&\makecell{20$\pm$0.0} \\
        woodworking-sat11-strips (20)&\makecell{\textbf{20$\pm$0.0}}&\makecell{\textbf{20}}& \makecell{12$\pm$1.1}&\makecell{\textbf{20}}&\makecell{\textbf{20}}&\makecell{\textbf{20}}&\makecell{\textbf{20$\pm$0.0}}&\makecell{\textbf{20$\pm$0.0}}&\makecell{\textbf{20$\pm$0.0}} \\ 
        \hline
        \rule{0pt}{8pt}\textbf{Coverage (1831)} &\makecell{1600$\pm$3.9}& \makecell{1603}& \makecell{1606$\pm$3.9}& \makecell{1535}& \makecell{1590}& \makecell{1626}& \makecell{1641$\pm$1.9} &\makecell{1662$\pm$2.3} &\makecell{\textbf{1688$\pm$3.3}}\\ \hline
        \rule{0pt}{12pt}\textbf{\% Score (100\%)} &\makecell{83.32\%\\$\pm$0.18}& \makecell{83.23\%}& \makecell{83.51\%\\$\pm$0.27}& \makecell{79.06\%}& \makecell{82.84\%}& \makecell{85.31\%}& \makecell{86.23\%\\$\pm$0.09}& \makecell{87.87\%\\$\pm$0.17} &\makecell{\textbf{89.79\%}\\\textbf{$\pm$0.22}}\\ \hline
        \rule{0pt}{8pt}\textbf{Frontend \% coverage share} &\makecell{-}& \makecell{-}& \makecell{-}& \makecell{-}& \makecell{-}& \makecell{-}&\makecell{97\%} &\makecell{96\%} &\makecell{94\%}\\ \hline
    \end{tabular}
    \vspace{1mm}
    \caption{Comparative performance analysis across various domains. \textit{\% score} is the average of the \% of instances solved in each domain. \textit{Frontend \% coverage share} refers to the \% of covered instances solved by the BFNoS frontend. Domains which are fully solved by all planners are omitted. Values for BFNoS variants and Approximate-BFWS represent the mean and include the standard deviation across 5 measurements. 
    Domains that are fully solved by all planners are omitted but included in Table~\ref{tab:comparative_performance_large}.
    }
    \vspace{2mm}
    \label{tab:comparative_performance}
\end{table*}

\subsection{Hybrid solvers with BFNoS frontend}
\label{subsect:hybrid-solvers}

We adopt $\text{BFNoS}_t(f_5(C_1),f_5(W_2))$ as a frontend solver, capped by a 6 $GB$ memory threshold and a time threshold close to the overall time limit, to enable backend fallback for all unresolved searches. We pair it with three backend planners from literature: the Dual-BFWS backend component (\textit{BFNoS-Dual}), LAMA-First (\textit{BFNoS-LAMA}), and the "first" version of Scorpion-Maidu (\textit{BFNoS-Maidu-$h^2$}), in its IPC2023 configuration with the \textit{$h^2$-preprocessor} \citep{alcazar2015reminder}. These were chosen for their complementary heuristics to our frontend's $f_5$ partitioning, promoting diverse solution strategies to enhance coverage diversity. Frontend time thresholds are set to 1600 \textit{sec} with BFNoS-Dual and BFNoS-LAMA, and 1400 \textit{sec} for BFNoS-Maidu-$h^2$, to account for up to 180 $sec$ of preprocessing allowance.

\begin{figure}[tbh]
    \centering
    \includegraphics[width=0.6\linewidth]{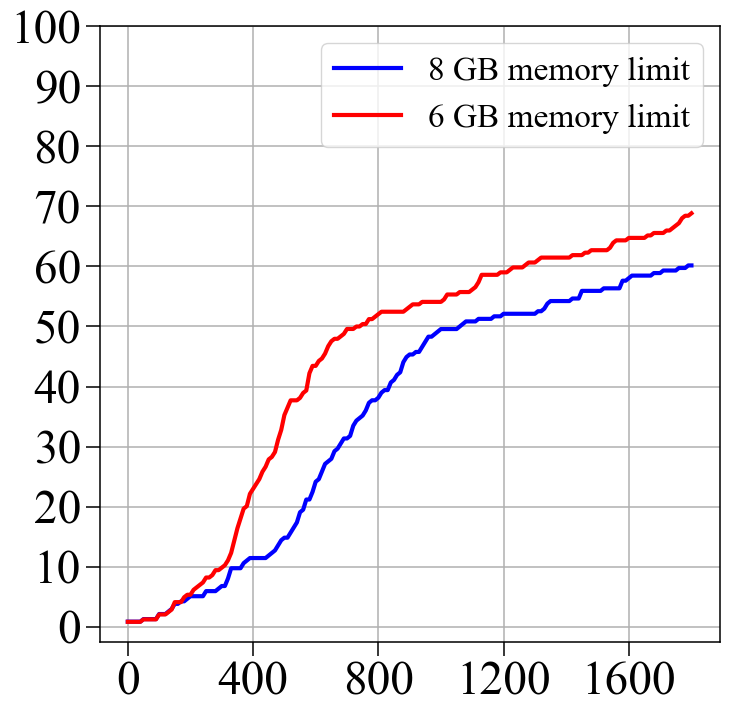}
    \caption{Cumulative \% of all failures attributed to search-time memory failures ($y$-axis) vs. time of failure (\textit{sec}) ($x$-axis), for BFNoS solvers with 6 \textit{GB} and 8 \textit{GB} memory and 1800 \textit{sec} time limits.}
    \label{fig:draft_memory_fail_by_time.png}
    \vspace{4mm}
\end{figure}

\paragraph{Memory threshold}
Solvers tend to exhibit diminishing returns in the number of instances solved with respect to both time and memory. Thus, as memory usage increases while solving an instance, it is more likely that the instance will not be solvable within the defined memory constraint, indicating that the adopted heuristics are not effective for that particular problem. Novelty heuristics $W_2$ and $C_1$ provide a quick evaluation of state novelty which, in combination with efficient partition functions, allow solvers to expand considerably more nodes per unit time than counterparts adopting more informed but expensive heuristics such as $h_{ff}$ \citep{hoffmann2001ff}, albeit at a greater memory cost. In other words, they can reach the "flatter" region of the coverage-memory relation more quickly. We leverage this trait to introduce dual configuration planners in which the frontend seeks to prioritize coverage, but also fail as quickly as possible when memory usage reaches those flatter areas, accounting for different domains' characteristics with respect to memory usage. Only search-time memory usage is considered, as tested dual solvers do not fallback to the backend when grounder errors occur. At 8 \textit{GB} and 1800 \textit{sec} limits, BFNoS hits half of its total failures by reaching the memory limit when half of the available time has passed, with memory failures constituting over 60\% of the total (Fig.~\ref{fig:draft_memory_fail_by_time.png}), making it a suitable frontend solver for our strategy.

\paragraph{Discussion}

Table~\ref{tab:comparative_performance} shows that $\text{BFNoS}_t(f_5(C_1),f_5(W_2))$ alone achieves comparable coverage to all baselines with the exception of the IPC2023 configuration of Scorpion-Maidu, which is the only baseline solver incorporating a preprocessor. 

The hybrid solvers adopting a $\text{BFNoS}_t(f_5(C_1),f_5(W_2))$ frontend outperform all baselines, denoting a meaningful increase in coverage and normalized \% score compared to their respective backends, particularly for the hybrid configuration with LAMA-First backend, which covers 62 and 127 more instances compared to BFNoS and LAMA-First respectively. BFNoS-Maidu-$h^2$ gains an edge on BFNoS-LAMA mainly thanks to the additional preprocessor in the backend, which only gets used if the frontend solver fails. Still, we note that our implementation of BFNoS-LAMA and BFNoS-Maidu-$h^2$ perform the grounding operation a second time for the backend, resulting in redundant computations which may use up significant time in hard-to-ground domains.

Our proposed frontend solver is the main component of all hybrid planners, being responsible for over 94\% of all solved instances in all cases. We note that the combination with the Dual-BFWS backend performs more poorly than other hybrid configurations, yet this does not come as a surprise. The Dual-BFWS backend is a Novelty planner which shares many more similarities with our own proposed frontend, most crucially the $W_2$ primary heuristic, thus causing more overlap in coverage between the frontend and backend. Still, even this version suffices to outperform tested benchmarks. We also note a significant increase in coverage of proposed hybrid planners on problem sets from IPC2023 and IPC2018, with BFNoS-LAMA covering 101 and 164 instances, and BFNoS-Maidu-$h^2$ covering 98 and 166 instances on average respectively. 

In all cases, the implementation of hybrid planner configurations leveraging only two powerful solvers allows us to keep such dual solvers as simple as possible, which we argue helps us reduce the risk of overfitting to the set of available benchmarks.

\begin{figure}[tbh]
    \centering
    \includegraphics[width=0.9\linewidth]{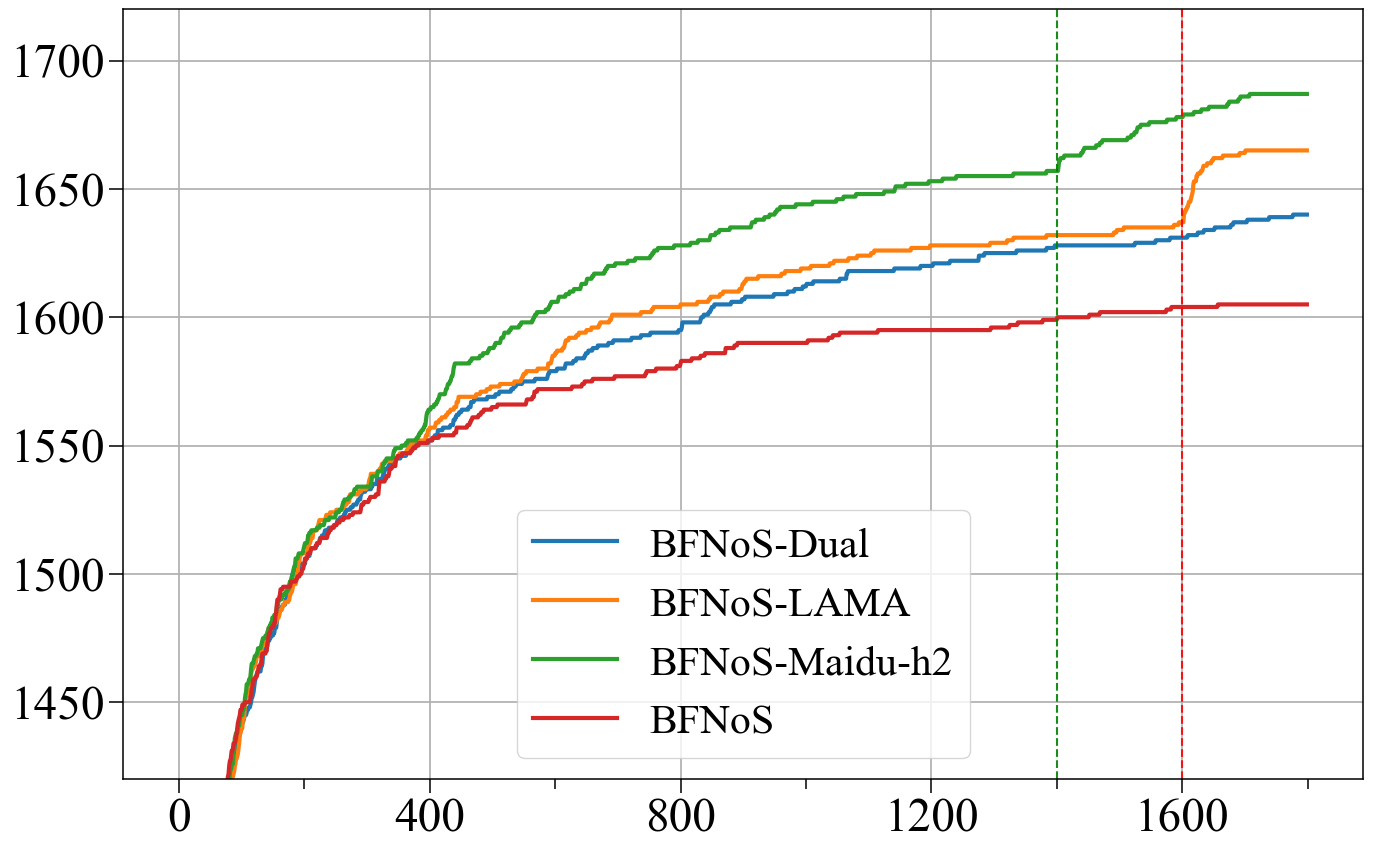}
    \vspace{1mm}
    \caption{Instance coverage ($y$-axis) vs. time (\textit{sec}) ($x$-axis). Comparison of BFNoS with the three presented hybrid configurations. The solvers are run in an identical configuration to previous sections, with BFNoS using the full 8 \textit{GB} memory allowance when run on its own, and a 6 \textit{GB} memory threshold when run as the frontend of hybrid configurations. The vertical lines signal the 1400 \textit{sec} (green) and 1600 \textit{sec} (red) time thresholds.}
    \label{fig:coverage-vs-time}
\end{figure}

\begin{table}[b]
  \centering
  \resizebox{\columnwidth}{!}{
  \begin{tabular}{lcc}
    \hline 
    \textbf{Hybrid Solver} & \textbf{Coverage} & \textbf{\% Score}  \\
    \hline
    {BFNoS-Dual$_{MO}$}& \makecell{1636$\pm$3.3} & \makecell{85.90\%$\pm$0.14} \\
    {BFNoS-LAMA$_{MO}$}& \makecell{1638$\pm$4.9} & \makecell{86.13\%$\pm$0.26}\\
    {BFNoS-Maidu$_{MO}$}& \makecell{1643$\pm$3.7}  & \makecell{86.45\%$\pm$0.17}\\
    {BFNoS-Maidu-$h^2$$_{MO}$}& \makecell{1662$\pm$4.7} & \makecell{88.02\%$\pm$0.21}\\
    \hline
  \end{tabular}
  }
  \vspace{1mm}
  \caption{
  Coverage and \% score of hybrid planner variants adopting only a 6 \textit{GB} memory threshold across all domains. Subscript $MO$ refers to memory-threshold-only versions. The BFNoS-Maidu planner is a variant of BFNoS-Maidu-$h^2$ that does not use the $h^2$ preprocessor with its Scorpion-Maidu backend. 
  Values represent the mean and standard deviation across 5 measurements.
  }
  \label{tab:hybrid-no-time-th}
   \vspace{2mm}
\end{table}

\vspace{4mm}
\paragraph{Coverage-over-time considerations}
We highlight the coverage-over-time benefits of adopting our proposed memory threshold strategy alongside BFNoS. The memory-threshold fallbacks of our proposed frontend start rapidly increasing shortly before the 400 \textit{sec} mark (Fig.~\ref{fig:draft_memory_fail_by_time.png}), and this fact is reflected in the increasing gap in instance coverage between the base BFNoS solver and dual configurations in Figure~\ref{fig:coverage-vs-time} beginning at that point in time. This trait is beneficial both for the average planning time of covered instances, as well as for satisficing planners, which gain more time to improve the quality of discovered plans. In this respect, memory-threshold-only variants of proposed dual solvers (Table~\ref{tab:hybrid-no-time-th}) $-$ that forego the use of time thresholds altogether $-$ still perform competitively compared to benchmark solvers in Table~\ref{tab:comparative_performance}. However, there still are limitations to the technique, as exemplified by the \textit{Rubik's Cube} domain \citep{muppasani2024rubiks} from IPC 2023. BFNoS fails to solve 15 out of the 20 instances available, but does not reach the 6 \textit{GB} memory threshold before the time limit, precluding fallback in memory-threshold-only dual configurations. This domain is the main contributor to the leap in instance coverage for BFNoS-LAMA and BFNoS-Maidu-$h^2$ that follows their respective time thresholds in Figure~\ref{fig:coverage-vs-time}.

\section{Conclusion}
In this paper we introduce the concept of count-based novelty as an alternative novelty exploration framework in Classical Planning, showing that the arity-1 variant of our proposed metric is capable of effectively predicting states with novel \textit{k}-tuples only using a constant number of tuples. We introduce the use of counts and Hamming distances to relate the exploratory behavior of count-based novelty to the existing body of knowledge on Novelty. We also propose single and double Trimmed Open List variants which allow us to upper bound the open list size by pruning nodes unlikely to be expanded. Our techniques used in the BFNoS solver demonstrate the effectiveness of combining distinct novelty metrics, achieving competitive coverage compared to state-of-the-art planners. Finally, we detail a dual-configuration strategy adopting a BFNoS frontend solver and introducing memory thresholds alongside customary time thresholds, justify the suitability of our proposed strategy, and demonstrate improved coverage performance compared to state-of-the-art planners.

Our work provides foundational knowledge on count-based novelty through basic solutions that mirror our theoretical analysis. 
Future directions may include alternative count-based variants over tuples or extracted features to guide exploration in general and domain-specific planners, as well as adaptation to the existing body of work blending Novelty and heuristic estimates. Our contributions also provide the basis to bridge Classical Planning with the broader paradigm of count-based exploration, benefiting knowledge transfer with related areas such as Reinforcement Learning. Trimmed open lists and memory thresholds also proved to be simple yet effective solutions, pointing at the potential for a deeper analysis of similar ideas in Classical Planning.


\begin{ack}
This research was supported by use of the Nectar Research Cloud and by the Melbourne Research Cloud. The Nectar Research Cloud is a collaborative Australian research platform supported by the NCRIS-funded Australian Research Data Commons (ARDC). We extend our gratitude to Anubhav Singh for providing the PDDL file modifications for the \textit{Tidybot}, \textit{Storage}, and \textit{GED} domains used with the Tarski grounder, and to Augusto B. Corrêa for providing the information needed to run the `First' version of Scorpion-Maidu using Fast Downward.

\end{ack}



\bibliography{paper}

\appendix
\section{Experimental details}


\paragraph{Hybrid configurations}
We set the memory threshold for the frontend BFNoS solver by performing memory usage measurements directly in LAPKT \citep{lapkt}, to then fallback to the backend solver. For the BFNoS-Dual solver, the fallback planner is ran as part of the same process in LAPKT. For the BFNoS-LAMA and BFNoS-Maidu-$h^2$ solvers, a script runs the BFNoS frontend first in LAPKT. When a threshold is reached, the LAPKT process is halted sending a suitable signal, such that the script may then run the backend. We pay special attention to ensure that the appropriate time limit is given to the backend process, to ensure that the overall dual configuration does not exceed the global 1800 \textit{sec} and 8 \textit{GB} limits when running our experiments with Downward Lab \citep{seipp2017downward}.

\paragraph{Multiple measurements} Multiple measurements were conducted for variants implementing trimmed open lists, as well as for the Approximate-BFWS solver \citep{singh2021technical}. We always selected the same set of seeds \{0,1,2,3,4\} over 5 measurements to avoid bias.

\paragraph{IPC benchmark domains for Approximate-BFWS}
In order to perform experiments over the full set of benchmark problems with Approximate-BFWS configuration using the Tarski grounder \cite{tarski:github:18}, we modified the description of the IPC domains \textit{Storage}, \textit{Tidybot}, and \textit{GED}, according to the changes described in \citet{singh2021technical}. 

\paragraph{First version of IPC2023 Scorpion-Maidu}
We run a "first" version of Scorpion Maidu \citep{correa-et-al-ipc2023c}, which halts after finding a solution rather than improving the plan, from the IPC-2023 branch of the code base \citep{maidu2023ipc} using the following command, as there is no explicit alias to directly run a "first" version of the planner:
\begin{verbatim}
--evaluator 'hlm=lmcount(
lm_factory=lm_reasonable_orders_hps(lm_rhw()),
transform=adapt_costs(one),pref=false)' 
--evaluator 
'hff=ff(transform=adapt_costs(one))' 
--search 'lazy(alt([single(hff), 
single(hff, pref_only=true),
single(hlm), single(hlm, pref_only=true), 
type_based([hff, g()]), 
novelty_open_list(novelty(width=2, 
consider_only_novel_states=true, 
reset_after_progress=True), 
break_ties_randomly=False, 
handle_progress=move)],
boost=1000),preferred=[hff,hlm], 
cost_type=one,reopen_closed=false)'
\end{verbatim}

\noindent This command is just for the planner, and does not include the $h^2$ preprocessing step.

\paragraph{Source code} 
Our implementation of the BFNoS solver is included in \citep{lapkt-bfnos}.

\section{Average Hamming distance measurement through simulation}

We simulate a search to obtain average Hamming distance values for newly generated nodes with respect to all nodes generated previously which are realistic and reflective of a true search tree. The underlying structural characteristics of planning state-spaces may vary significantly depending on the domain. Performing a simulated search allows us to simplify the environment as well as control certain characteristics of the environment so as to not bias it towards the structure of specific domains. Such control parameters include the number of nodes generated, the number of children nodes generated from each parent expansion, and the means of generating children nodes, which are generated from a parent node by selecting $J$ variables uniformly at random and changing their boolean values. 

We set the following parameters, which are representative of corresponding measures found in IPC satisficing benchmark domains: state size $L=100$, each parent state generates 4 children states, by changing the value of 3 boolean variables selected uniformly at random. The search starts with only one boolean vector, generated uniformly at random, in the open list, which gets removed from the open list and expanded, by generating its successors as described. Each newly generated node is inserted at the back of the open list. The new node to expand is then the first element of the open list, which then gets removed and the procedure repeated. Nodes which are expanded are inserted into a closed list. We proceed to generate 10000 nodes in this manner, and measure the average Hamming distance of each one of the last 100 generated nodes with respect to all other generated nodes, and average our expected average Hamming distance value across these measurements. The expected normalized Hamming distance $\alpha(s)$ is then obtained by dividing this value by the state size $L=100$. 

We obtained average Hamming distances of $\approx 35$, which results in an average normalized Hamming distance of $\frac{35}{100}=0.35$. We rounded this value down to $0.3$ to account for a slightly more pessimistic case than our simulation.

\section{Extended proofs}

\noindent \textbf{Theorem 4.}
\textit{Lower and upper bounds for $\alpha_{0:t}(n^c)$ are given by:}

\[\alpha_{0:t}(n^p)-\frac{t-1}{t}\frac{e}{L} \leq \alpha_{0:t}(n^c) \leq \alpha_{0:t}(n^p)+\frac{t-1}{t}\frac{e}{L}\]

\begin{proof}
When comparing the Hamming distances of $n^p$ and $n^c$ with respect to a third node $n'$, the greatest decrease in Hamming distances $e$, corresponding to all $e$ effects changing variables $v_i$ where $n^p(v_i)\neq n'(v_i)$ and $n^c(v_i)=n'(v_i)$. Thus in the lower bound we get that all $e$ effect variables change their corresponding valuation to match with all states in history except for the parent node, reducing Hamming distance to each state by 1 for each effect $e$. The upper bound is symmetric, and we take into account the fact that the distance of $n^c$ to $n^p$ is already accounted as the distance of $n^p$ to $n^c$ in $\alpha_{0:t}(n^p)$, since $n^c=n_t$, and it thus does not change. Since parent and child states share all variable valuations except for $e$ effects, which change valuation from parent to child node. This yields,  for all cases where $n'\in n_{0:t}$, $n'\neq n^c$ and $n'\neq n^p$:
\begin{equation} \label{Aeq2-1}
    \delta(n^c, n') \geq \frac{1}{L} (H(n^p, n') - e) = \delta(n^p, n') - \frac{e}{L}
\end{equation} 
\begin{equation} \label{Aeq2-2}
    \delta(n^c, n^p)=\frac{1}{L} (H(n^p, n^p) + e) = \frac{e}{L}
\end{equation}

We can redefine the average $\alpha_{0:t}(n^c)$ as 
\begin{equation} \label{Aeq2-3}
    \alpha_{0:t}(n^c)=\frac{1}{t}\bigg[\sum_{i=0;n_i \notin \{n^p,n^c\}}^t\Big(\delta(n^c, n_i)\Big) + \delta(n^c,n^p)\bigg]
\end{equation}
We have that, for $t-1$ comparisons, $\sum_{j=1}^{t-1} \alpha_{0:t-1}(n^p)= \sum_{j=1}^{t-1} \frac{1}{t-1} \sum_{i=0; n_i \neq n^p}^{t-1} \delta(n^p, n_i) = \sum_{i=0; n_i \neq n^p}^{t-1} \delta(n^p, n_i)$, therefore we can update the average for $n^p$ which includes node $n^c=n_t$:
\begin{equation} \label{Aeq2-4}
    \begin{split}
    \alpha_{0:t}(n^p)&=\frac{1}{t}\bigg[\sum_{j=1}^{t-1} \alpha_{0:t-1}(n^p)+\delta(n^c,n^p)\bigg] \\
    &=\frac{1}{t}\bigg[\sum_{i=0; n_i \neq n^p}^{t-1}\Big(\delta(n^p,n_i)\Big)+\delta(n^c,n^p)\bigg]
    \end{split}
\end{equation}
Substituting (\ref{Aeq2-1}) into (\ref{Aeq2-3}), noting that \(\sum_{i=0;n_i \notin \{n^p,n^c\}}^t(\frac{e}{L})=(t-1)\frac{e}{L}\), and that \(\sum_{i=0;n_i \notin \{n^p,n^c\}}^t\big(\delta(n^p, n_i)\big)=\sum_{i=0;n_i \notin \{n^p\}}^{t-1}\big(\delta(n^p, n_i)\big)\) since $n^c=n_t$, we obtain:
\begin{equation}
    \begin{split} \label{Aeq2-5}
    \alpha_{0:t}(n^c) &\geq \frac{1}{t}\bigg[\sum_{i=0;n_i \notin \{n^p,n^c\}}^t\Big(\delta(n^p, n_i) - \frac{e}{L}\Big) + \delta(n^c,n^p)\bigg] \\
    &=\frac{1}{t}\bigg[\sum_{i=0;n_i \notin \{n^p\}}^{t-1}\Big(\delta(n^p, n_i)\Big) + \delta(n^c,n^p)\bigg]-\frac{t-1}{t}\frac{e}{L}
    \end{split}
\end{equation}

Substituting (\ref{Aeq2-4}) into (\ref{Aeq2-5}) we obtain 
\begin{equation}
    \alpha_{0:t}(n^c)\geq\alpha_{0:t}(n^p)-\frac{t-1}{t}\frac{e}{L}
\end{equation}

For the upper bound, we note that it is symmetrical in that in the upper bound all effects \textit{e} are novel, that is for some effect variable $v_i$ we have that $n^c(v_i)\neq n'(v_i)$ for all $n'\in n_{0:t-1}$, thus we get $\delta(n^c,n')\leq\delta(n^p,n')+\frac{e}{L}$. Following the same procedure yields the upper bound.
\end{proof}

\begin{example}
For states of size $L=3$, we give a history of \textit{t} nodes, one of which the parent. All $t-1$ nodes that are not the parent node $n^p$ are represented by state vector [1,0,1]. Parent node $n^p$ is represented by [1,1,0]. For an action with 2 effects, $e=2$, knowing that the Hamming distance between parent and child node must be $e=2$ by assumption, then the greatest decrease occurs when child node $n^c$ also has value [1,0,1]. Parent node $n^p$ has average normalized Hamming distance of $\frac{2}{3}$ to all other nodes, where 3 is given by $L=3$. Child node $n^c$ has average normalized Hamming distance of $\frac{1}{t}\frac{2}{3}=\frac{t}{t}\frac{2}{3}-\frac{t-1}{t}\frac{2}{3}=\alpha_{0:t}(n^p)-\frac{t-1}{t}\frac{2}{L}$.
\end{example}

\noindent \textbf{Theorem 5.}
\textit{Upper bound $\alpha_{0:t}(n^c)$ with respect to $\mu = \mu_{t-1}^{min}(n^c)$ is given by:}
\[
\alpha_{0:t}(n^c)\leq\alpha_{0:t}(n^p)+\frac{t-1}{t}\frac{e-2e\mu}{L}
\]

\begin{proof}
We seek to maximize the Hamming distance of the child node with respect to its parent by minimizing the number of nodes in history $n'\in n_{0:t-1}$ where value $n^c(v_i)=n'(v_i)$ for effect variables $v_i$. Since $\mu$ is the minimum feature occurrence, this acts as a constraint, and the upper bound occurs when all effects $e$ have occurrence equal to the minimum occurrence $\mu$. Thus there are $(1-\mu)\cdot (t-1)$ nodes in which, for effect variables $v_i$, $n'(v_i)\neq n^c(v_i)$ and $n'(v_i)=n^p(v_i)$, and $\mu\cdot (t-1)$ nodes in which $n'(v_i)=n^c(v_i)$ and $n'(v_i)\neq n^p(v_i)$. Thus, for $t-1$ nodes $n'\in n_{0:t-1}$, $n'\neq n^p$, we have that $\delta(n^c, n')=\frac{1}{L} (H(n^p, n') + e)$ a total of $(1-\mu)\cdot (t-1)$ times, and $\delta(n^c, n')=\frac{1}{L} (H(n^p, n') - e)$ a total of $\mu\cdot (t-1)$ times.

The summation over all comparisons becomes:

\begin{equation} \label{Aeq2-7}
    \begin{split}
        &\sum_{i=0;n_i \notin \{n^p,n^c\}}^t\Big(\delta(n^c, n_i)\Big) \\
        &=\sum_{i=0;n_i \notin \{n^p,n^c\}}^t\Big(\delta(n^p, n_i)\Big)+(1-\mu)(t-1)\frac{e}{L} - \mu(t-1)\frac{e}{L} \\
        &=\sum_{i=0;n_i \notin \{n^p\}}^{t-1}\Big(\delta(n^p, n_i)\Big)+(t-1)\cdot \frac{e-2e\mu}{L}
    \end{split}
\end{equation}

Substituting (\ref{Aeq2-7}) into (\ref{Aeq2-3}) as in Theorem 4 we obtain:

\begin{equation} \label{Aeq2-8}
    \begin{split} 
    &\alpha_{0:t}(n^c) \\
    &\leq\frac{1}{t}\bigg[\sum_{i=0;n_i \notin \{n^p\}}^{t-1}\Big(\delta(n^p, n_i)\Big)+(t-1)\cdot \frac{e-2e\mu}{L} + \delta(n^c,n^p)\bigg] \\
    &=\frac{1}{t}\bigg[\sum_{i=0;n_i \notin \{n^p\}}^{t-1}\Big(\delta(n^p, n_i)\Big) + \delta(n^c,n^p)\bigg]+\frac{t-1}{t}\frac{e-2e\mu}{L}
    \end{split}
\end{equation}

Substituting (\ref{Aeq2-4}) into (\ref{Aeq2-8}) we obtain:
\[
\alpha_{0:t}(n^c)\leq\alpha_{0:t}(n^p)+\frac{t-1}{t}\frac{e-2e\mu}{L}
\]

\end{proof}

\noindent \textbf{Theorem 6.}
\textit{Lower bound for $\alpha_{0:t}(n^c)$ when $\mu_{t-1}^{min}(n^c)=0$ is given by:}
\[
\alpha_{0:t}(n^c) \geq \alpha_{0:t}(n^p)-\frac{t-1}{t}\frac{e-2}{L}
\]
\begin{proof}
In the lower bound, one effect of the action from parent to child node is constrained to be novel, resulting in a Hamming distance of $+1$ compared to the parent node, and $e-1$ effects match all previous history except $n^p$, resulting in a hamming distance of $-1$ compared to the parent. Thus we have that:
\begin{equation} \label{Aeq2-9}
    \begin{split}
        \delta(n^c, n') &= \frac{1}{L} (H(n^p, n') - (e-1) + 1) \\
        &= \delta(n^p, n') - \frac{e-2}{L}
    \end{split}
\end{equation}
Inserting (\ref{Aeq2-9}) into (\ref{Aeq2-3}) and following the derivation from Theorem 4 yields Theorem 6.
\end{proof}

\section{Extended results}

All measurements in this section are subject to an 8 \textit{GB} memory and 1800 \textit{sec} time limits.

\vspace{2mm}
\begin{table}[tbh]
  \centering
  \resizebox{\columnwidth}{!}{
  \begin{tabular}{l|ccccccc}
    \hline
    \rule{0pt}{2.6ex} & \textbf{$N\geq0$} & \textbf{$N\geq1$} & \textbf{$N\geq5$} & \textbf{$N\geq10$} & \textbf{$N\geq100$} & \textbf{$N\geq1000$} & \textbf{$N\geq10000$}\\ 
    \hline
    \rule{0pt}{2.7ex}\text{Generated} &100\%& 99.45\% & 97.68\% & 95.09\% & 74.61\% & 45.35\% & 28.26\% \\
    \text{Expanded} & \makecell{100\%} & \makecell{34.88\%} & \makecell{27.21\%} & \makecell{24.45\%} & \makecell{16.34\%} & \makecell{10.49\%} & \makecell{4.91\%}\\
    \hline
  \end{tabular}}
  \vspace{1mm}
  \caption{\% of instances across all IPC satisficing benchmarks where a node with count $\geq N$ was recorded across generated and expanded nodes by a $\text{BFCS}_t(f_5)$ planner. This includes unsolved instances.}
  \label{tab:reformatted}
\end{table}

\begin{figure}[tbh]
    \centering
    \begin{overpic}[width=0.6\linewidth]{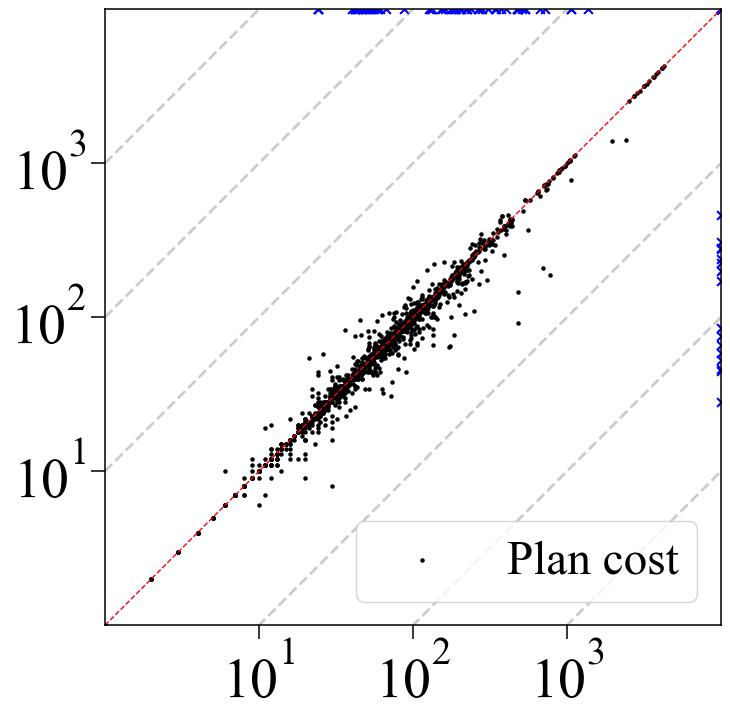}
        \put(55,-1){\makebox(0,0){BFNoS}} 
        \put(-1,55){\rotatebox{90}{\makebox(0,0){BFWS$_t(f_5)$}}} 
    \end{overpic}
    \vspace{3mm} 
    \caption{Plan cost over instances solved by $\text{BFNoS}$ and $\text{BFWS}_t(f_5)$. Blue crosses represent instances not solved by at least one planner.}
\end{figure}

\begin{figure}[tbh]
    \centering
    \includegraphics[width=0.8\linewidth]{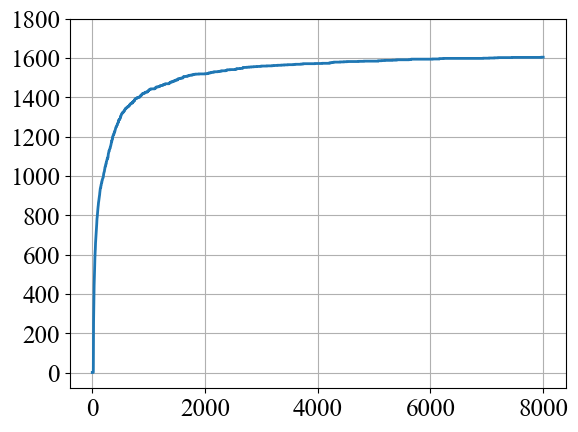}
    \caption{Instance coverage ($y$-axis) vs. memory usage (\textit{MB}) ($x$-axis) for BFNoS. The curve follows the diminishing returns on instances covered per memory usage claimed in Section~\ref{subsect:hybrid-solvers}. The problem set contains a total of 1831 instances.}
\end{figure}

\begin{figure}[t]
    \centering
    \begin{overpic}[width=0.6\linewidth]{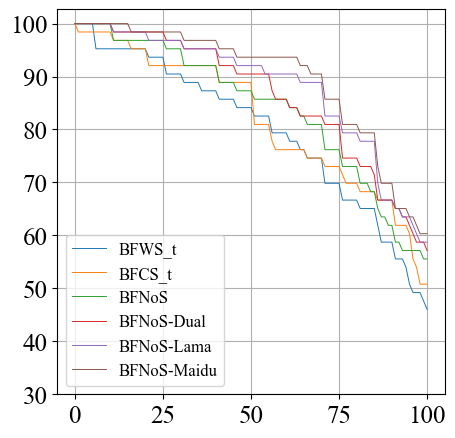}
        \put(55,-1){\makebox(0,0){\% instances solved}} 
        \put(0,48){\rotatebox{90}{\makebox(0,0){\% domains}}} 
    \end{overpic}
    \vspace{4mm}
    \caption{Coverage analysis of proposed solvers. Lines outline the \% of benchmark domains where the $\%$ of domain instances solved is $\geq$ a given value.}
\end{figure}

\clearpage

\begin{table*}[tbhp]
    \centering
    \begin{tabular}{|l|c|c|c|c|c|c|c|c|c|} \hline
        Domain & BFNoS & Dual- & Apx-BFWS &  LAMA- & Maidu &  Maidu &  BFNoS- & BFNoS- & BFNoS- \\
        & & BFWS & (Tarski) & First &  & with $h^2$ & Dual-back & LAMA & Maidu-${h^2}$\\ \hline
        agricola-sat18-strips &\makecell{15$\pm$0.0}&\makecell{12}&\makecell{{18$\pm$0.9}}&\makecell{12}&\makecell{12}&\makecell{13}&\makecell{15$\pm$0.5}&\makecell{15$\pm$0.5}&\makecell{15$\pm$0.5} \\ 
        airport &\makecell{{47$\pm$0.6}}&\makecell{46}&\makecell{{47$\pm$0.6}}&\makecell{34}&\makecell{38}&\makecell{45}&\makecell{46$\pm$0.6}&\makecell{46$\pm$0.5}&\makecell{46$\pm$0.6} \\
        assembly &\makecell{30$\pm$0.0}&\makecell{30}&\makecell{30$\pm$0.0}&\makecell{30}&\makecell{30}&\makecell{30}&\makecell{30$\pm$0.0}&\makecell{30$\pm$0.0}&\makecell{30$\pm$0.0} \\
        barman-sat14-strips &\makecell{20$\pm$0.0}&\makecell{20}&\makecell{20$\pm$0.0}&\makecell{20}&\makecell{20}&\makecell{20}&\makecell{20$\pm$0.0}&\makecell{20$\pm$0.0}&\makecell{20$\pm$0.0} \\ 
        blocks &\makecell{35$\pm$0.0}&\makecell{35}&\makecell{35$\pm$0.0}&\makecell{35}&\makecell{35}&\makecell{35}&\makecell{35$\pm$0.0}&\makecell{35$\pm$0.0}&\makecell{35$\pm$0.0} \\
        caldera-sat18-adl &\makecell{18$\pm$0.0}&\makecell{{19}}&\makecell{{19$\pm$0.6}}&\makecell{16}&\makecell{16}&\makecell{16}&\makecell{16$\pm$0.0}&\makecell{17$\pm$0.5}&\makecell{18$\pm$0.0} \\ 
        cavediving-14-adl &\makecell{{8$\pm$0.5}}&\makecell{{8}}&\makecell{{8$\pm$0.5}}&\makecell{7}&\makecell{7}&\makecell{7}&\makecell{{8$\pm$0.0}}&\makecell{{8$\pm$0.0}}&\makecell{{8$\pm$0.5}} \\
        childsnack-sat14-strips &\makecell{1$\pm$1.1}&\makecell{{9}}&\makecell{5$\pm$0.6}&\makecell{6}&\makecell{6}&\makecell{6}&\makecell{8$\pm$0.0}&\makecell{6$\pm$0.0}&\makecell{6$\pm$0.5} \\ 
        citycar-sat14-adl &\makecell{{20$\pm$0.0}}&\makecell{{20}}&\makecell{{20$\pm$0.0}}&\makecell{5}&\makecell{6}&\makecell{6}&\makecell{{20$\pm$0.0}}&\makecell{{20$\pm$0.0}}&\makecell{{20$\pm$0.0}} \\ 
        data-network-sat18-strips &\makecell{17$\pm$0.6}&\makecell{13}&\makecell{{19$\pm$0.5}}&\makecell{13}&\makecell{16}&\makecell{16}&\makecell{16$\pm$0.8}&\makecell{15$\pm$1.1}&\makecell{16$\pm$0.8} \\ 
        depot &\makecell{{22$\pm$0.0}}&\makecell{{22}}&\makecell{{22$\pm$0.0}}&\makecell{20}&\makecell{{22}}&\makecell{{22}}&\makecell{{22$\pm$0.0}}&\makecell{{22$\pm$0.0}}&\makecell{{22$\pm$0.0}} \\ 
        driverlog &\makecell{20$\pm$0.0}&\makecell{20}&\makecell{20$\pm$0.0}&\makecell{20}&\makecell{20}&\makecell{20}&\makecell{20$\pm$0.0}&\makecell{20$\pm$0.0}&\makecell{20$\pm$0.0} \\
        elevators-sat11-strips &\makecell{20$\pm$0.0}&\makecell{20}&\makecell{20$\pm$0.0}&\makecell{20}&\makecell{20}&\makecell{20}&\makecell{20$\pm$0.0}&\makecell{20$\pm$0.0}&\makecell{20$\pm$0.0} \\ 
        flashfill-sat18-adl &\makecell{14$\pm$1.3}&\makecell{{17}}&\makecell{15$\pm$1.6}&\makecell{14}&\makecell{15}&\makecell{14}&\makecell{{17$\pm$0.5}}&\makecell{16$\pm$0.6}&\makecell{16$\pm$0.9} \\
        floortile-sat14-strips &\makecell{2$\pm$0.5}&\makecell{2}&\makecell{2$\pm$0.0}&\makecell{2}&\makecell{2}&\makecell{{20}}&\makecell{2$\pm$0.0}&\makecell{2$\pm$0.0}&\makecell{{20$\pm$0.0}} \\ 
        folding &\makecell{9$\pm$0.0}&\makecell{5}&\makecell{5$\pm$0.5}&\makecell{{11}}&\makecell{{11}}&\makecell{{11}}&\makecell{9$\pm$0.0}&\makecell{9$\pm$0.0}&\makecell{9$\pm$0.0} \\
        freecell &\makecell{{80$\pm$0.0}}&\makecell{{80}}&\makecell{{80$\pm$0.0}}&\makecell{79}&\makecell{{80}}&\makecell{{80}}&\makecell{{80$\pm$0.0}}&\makecell{{80$\pm$0.}0}&\makecell{{80$\pm$0.0}} \\ 
        ged-sat14-strips &\makecell{20$\pm$0.0}&\makecell{20}&\makecell{20$\pm$0.0}&\makecell{20}&\makecell{20}&\makecell{20}&\makecell{20$\pm$0.0}&\makecell{20$\pm$0.0}&\makecell{20$\pm$0.0} \\
        grid &\makecell{5$\pm$0.0}&\makecell{5}&\makecell{5$\pm$0.0}&\makecell{5}&\makecell{5}&\makecell{5}&\makecell{5$\pm$0.0}&\makecell{5$\pm$0.0}&\makecell{5$\pm$0.0} \\ 
        gripper &\makecell{20$\pm$0.0}&\makecell{20}&\makecell{20$\pm$0.0}&\makecell{20}&\makecell{20}&\makecell{20}&\makecell{20$\pm$0.0}&\makecell{20$\pm$0.0}&\makecell{20$\pm$0.0} \\ 
        hiking-sat14-strips &\makecell{{20$\pm$0.0}}&\makecell{18}&\makecell{{20$\pm$0.0}}&\makecell{{20}}&\makecell{{20}}&\makecell{{20}}&\makecell{{20$\pm$0.0}}&\makecell{{20$\pm$0.0}}&\makecell{{20$\pm$0.0}} \\ 
        labyrinth &\makecell{15$\pm$0.5}&\makecell{5}&\makecell{{18$\pm$0.5}}&\makecell{1}&\makecell{0}&\makecell{2}&\makecell{15$\pm$0.5}&\makecell{15$\pm$0.5}&\makecell{15$\pm$0.5} \\ 
        logistics00 &\makecell{28$\pm$0.0}&\makecell{28}&\makecell{28$\pm$0.0}&\makecell{28}&\makecell{28}&\makecell{28}&\makecell{28$\pm$0.0}&\makecell{28$\pm$0.0}&\makecell{28$\pm$0.0} \\ 
        maintenance-sat14-adl &\makecell{{17$\pm$0.0}}&\makecell{{17}}&\makecell{{17$\pm$0.0}}&\makecell{11}&\makecell{13}&\makecell{13}&\makecell{{17$\pm$0.0}}&\makecell{{17$\pm$0.0}}&\makecell{{17$\pm$0.0}} \\ 
        miconic &\makecell{150$\pm$0.0}&\makecell{150}&\makecell{150$\pm$0.0}&\makecell{150}&\makecell{150}&\makecell{150}&\makecell{150$\pm$0.0}&\makecell{150$\pm$0.0}&\makecell{150$\pm$0.0} \\ 
        movie &\makecell{30$\pm$0.0}&\makecell{30}&\makecell{30$\pm$0.0}&\makecell{30}&\makecell{30}&\makecell{30}&\makecell{30$\pm$0.0}&\makecell{30$\pm$0.0}&\makecell{30$\pm$0.0} \\ 
        mprime &\makecell{35$\pm$0.0}&\makecell{35}&\makecell{35$\pm$0.0}&\makecell{35}&\makecell{35}&\makecell{35}&\makecell{35$\pm$0.0}&\makecell{35$\pm$0.0}&\makecell{35$\pm$0.0} \\ 
        mystery &\makecell{18$\pm$0.6}&\makecell{{19}}&\makecell{{19$\pm$0.0}}&\makecell{{19}}&\makecell{{19}}&\makecell{{19}}&\makecell{{19$\pm$0.0}}&\makecell{{19$\pm$0.0}}&\makecell{{19$\pm$0.0}} \\ 
        nomystery-sat11-strips &\makecell{14$\pm$0.8}&\makecell{{19}}&\makecell{14$\pm$0.5}&\makecell{11}&\makecell{{19}}&\makecell{18}&\makecell{{19$\pm$0.0}}&\makecell{15$\pm$0.6}&\makecell{17$\pm$0.0} \\ 
        nurikabe-sat18-adl &\makecell{16$\pm$0.6}&\makecell{14}&\makecell{17$\pm$0.5}&\makecell{9}&\makecell{11}&\makecell{16}&\makecell{17$\pm$0.6}&\makecell{17$\pm$0.0}&\makecell{{18$\pm$0.0}} \\ 
        openstacks-sat14-strips &\makecell{20$\pm$0.0}&\makecell{20}&\makecell{20$\pm$0.6}&\makecell{20}&\makecell{20}&\makecell{20}&\makecell{20$\pm$0.0}&\makecell{20$\pm$0.0}&\makecell{20$\pm$0.0} \\ 
        organic-synthesis-split-sat18-strips &\makecell{8$\pm$0.5}&\makecell{12}&\makecell{8$\pm$0.0}&\makecell{{14}}&\makecell{{14}}&\makecell{{14}}&\makecell{11$\pm$0.5}&\makecell{{14$\pm$0.0}}&\makecell{{14$\pm$0.9}} \\ 
        parcprinter-sat11-strips &\makecell{9$\pm$0.6}&\makecell{16}&\makecell{11$\pm$1.3}&\makecell{{20}}&\makecell{{20}}&\makecell{{20}}&\makecell{{20$\pm$0.0}}&\makecell{{20$\pm$0.0}}&\makecell{{20$\pm$0.0}} \\ 
        parking-sat14-strips &\makecell{20$\pm$0.0}&\makecell{20}&\makecell{20$\pm$0.0}&\makecell{20}&\makecell{20}&\makecell{20}&\makecell{20$\pm$0.0}&\makecell{20$\pm$0.0}&\makecell{20$\pm$0.0} \\ 
        pathways &\makecell{26$\pm$0.9}&\makecell{{30}}&\makecell{28$\pm$1.1}&\makecell{23}&\makecell{25}&\makecell{25}&\makecell{{30$\pm$0.0}}&\makecell{27$\pm$0.8}&\makecell{27$\pm$0.7} \\ 
        pegsol-sat11-strips &\makecell{20$\pm$0.0}&\makecell{20}&\makecell{20$\pm$0.0}&\makecell{20}&\makecell{20}&\makecell{20}&\makecell{20$\pm$0.0}&\makecell{20$\pm$0.0}&\makecell{20$\pm$0.0} \\ 
        pipesworld-notankage &\makecell{{50$\pm$0.0}}&\makecell{{50}}&\makecell{{50$\pm$0.0}}&\makecell{43}&\makecell{45}&\makecell{45}&\makecell{50$\pm$0.0}&\makecell{50$\pm$0.0}&\makecell{50$\pm$0.0} \\ 
        pipesworld-tankage &\makecell{43$\pm$1.6}&\makecell{42}&\makecell{{44$\pm$0.6}}&\makecell{43}&\makecell{43}&\makecell{43}&\makecell{43$\pm$0.8}&\makecell{43$\pm$0.5}&\makecell{43$\pm$0.6} \\ 
        psr-small &\makecell{50$\pm$0.0}&\makecell{50}&\makecell{50$\pm$0.0}&\makecell{50}&\makecell{50}&\makecell{50}&\makecell{50$\pm$0.0}&\makecell{50$\pm$0.0}&\makecell{50$\pm$0.0} \\ 
        quantum-layout &\makecell{20$\pm$0.0}&\makecell{20}&\makecell{20$\pm$0.0}&\makecell{20}&\makecell{20}&\makecell{20}&\makecell{20$\pm$0.0}&\makecell{20$\pm$0.0}&\makecell{20$\pm$0.0} \\ 
        recharging-robots &\makecell{{14$\pm$0.6}}&\makecell{12}&\makecell{{14$\pm$0.8}}&\makecell{13}&\makecell{13}&\makecell{13}&\makecell{{14$\pm$0.5}}&\makecell{{14$\pm$0.0}}&\makecell{{14$\pm$0.5}} \\ 
        ricochet-robots &\makecell{{20$\pm$0.5}}&\makecell{{20}}&\makecell{18$\pm$0.6}&\makecell{14}&\makecell{18}&\makecell{18}&\makecell{{20$\pm$0.0}}&\makecell{{20$\pm$0.0}}&\makecell{{20$\pm$0.0}} \\ 
        rovers &\makecell{40$\pm$0.4}&\makecell{40}& \makecell{40$\pm$0.4}&\makecell{{40}}&\makecell{{40}}&\makecell{{40}}&\makecell{40$\pm$0.4}&\makecell{{40$\pm$0.0}}&\makecell{40$\pm$0.0} \\ 
        rubiks-cube &\makecell{5$\pm$0.0}&\makecell{6}& \makecell{5$\pm$0.6}&\makecell{{20}}&\makecell{{20}}&\makecell{{20}}&\makecell{5$\pm$0.0}&\makecell{{20$\pm$0.0}}&\makecell{16$\pm$0.6} \\ 
        satellite &\makecell{34$\pm$0.8}&\makecell{33}& \makecell{34$\pm$0.5}&\makecell{{36}}&\makecell{{36}}&\makecell{{36}}&\makecell{34$\pm$0.6}&\makecell{35$\pm$0.0}&\makecell{35$\pm$0.0} \\ 
        scanalyzer-sat11-strips &\makecell{20$\pm$0.0}&\makecell{20}& \makecell{20$\pm$0.5}&\makecell{20}&\makecell{20}&\makecell{20}&\makecell{20$\pm$0.0}&\makecell{20$\pm$0.0}&\makecell{20$\pm$0.0} \\ 
        schedule &\makecell{149$\pm$1.3}&\makecell{{150}}&  \makecell{149$\pm$1.3}&\makecell{{150}}&\makecell{{150}}&\makecell{{150}}&\makecell{149$\pm$0.7}&\makecell{{150$\pm$0.0}}&\makecell{{150$\pm$0.0}} \\ 
        settlers-sat18-adl &\makecell{13$\pm$1.5}&\makecell{7}& \makecell{12$\pm$0.7}&\makecell{17}&\makecell{{18}}&\makecell{{18}}&\makecell{12$\pm$0.5}&\makecell{17$\pm$0.0}&\makecell{17$\pm$0.5} \\ 
        slitherlink &\makecell{{5$\pm$0.6}}&\makecell{{5}}& \makecell{{5$\pm$0.7}}&\makecell{0}&\makecell{0}&\makecell{0}&\makecell{{5$\pm$0.5}}&\makecell{3$\pm$0.6}&\makecell{4$\pm$0.7} \\ 
        snake-sat18-strips &\makecell{{20$\pm$0.0}}&\makecell{17}& \makecell{{20$\pm$0.0}}&\makecell{5}&\makecell{14}&\makecell{14}&\makecell{{20$\pm$0.0}}&\makecell{{20$\pm$0.0}}&\makecell{{20$\pm$0.0}} \\ 
        sokoban-sat11-strips &\makecell{15$\pm$1.1}&\makecell{17}& \makecell{14$\pm$0.9}&\makecell{19}&\makecell{19}&\makecell{{20}}&\makecell{15$\pm$0.5}&\makecell{19$\pm$0.0}&\makecell{{20$\pm$0.0}} \\ 
        spider-sat18-strips &\makecell{17$\pm$1.3}&\makecell{16}& \makecell{16$\pm$1.1}&\makecell{16}&\makecell{16}&\makecell{17}&\makecell{{18$\pm$0.0}}&\makecell{{18$\pm$0.0}}&\makecell{{18$\pm$0.9}} \\ 
        storage &\makecell{{30$\pm$0.5}}&\makecell{29}& \makecell{{30$\pm$0.0}}&\makecell{20}&\makecell{25}&\makecell{25}&\makecell{29$\pm$0.5}&\makecell{29$\pm$0.0}&\makecell{29$\pm$0.6} \\ 
        termes-sat18-strips &\makecell{10$\pm$0.8}&\makecell{10}& \makecell{5$\pm$1.5}&\makecell{{16}}&\makecell{14}&\makecell{14}&\makecell{10$\pm$0.5}&\makecell{14$\pm$0.0}&\makecell{14$\pm$0.0} \\ 
        tetris-sat14-strips &\makecell{{20$\pm$0.0}}&\makecell{17}& \makecell{{20$\pm$0.0}}&\makecell{16}&\makecell{17}&\makecell{20}&\makecell{{20$\pm$0.0}}&\makecell{{20$\pm$0.0}}&\makecell{{20$\pm$0.0}} \\ 
        thoughtful-sat14-strips &\makecell{{20$\pm$0.0}}&\makecell{{20}}& \makecell{{20$\pm$0.2}}&\makecell{15}&\makecell{19}&\makecell{19}&\makecell{{20$\pm$0.0}}&\makecell{{20$\pm$0.0}}&\makecell{{20$\pm$0.0}} \\ 
        tidybot-sat11-strips &\makecell{{20$\pm$0.0}}&\makecell{18}& \makecell{{20$\pm$0.2}}&\makecell{17}&\makecell{{20}}&\makecell{{20}}&\makecell{{20$\pm$0.0}}&\makecell{{20$\pm$0.0}}&\makecell{{20$\pm$0.0}} \\ 
        tpp &\makecell{30$\pm$0.5}&\makecell{30}& \makecell{30$\pm$0.3}&\makecell{30}&\makecell{30}&\makecell{30}&\makecell{30$\pm$0.0}&\makecell{30$\pm$0.0}&\makecell{30$\pm$0.0} \\
        transport-sat14-strips &\makecell{{20$\pm$0.0}}&\makecell{{20}}& \makecell{{20$\pm$0.2}}&\makecell{17}&\makecell{18}&\makecell{16}&\makecell{{20$\pm$0.5}}&\makecell{{20$\pm$0.0}}&\makecell{{20$\pm$0.0}} \\
        trucks-strips &\makecell{8$\pm$0.8}&\makecell{19}& \makecell{13$\pm$1.5}&\makecell{18}&\makecell{20}&\makecell{{22}}&\makecell{17$\pm$0.5}&\makecell{16$\pm$0.0}&\makecell{20$\pm$0.0} \\ 
        visitall-sat14-strips &\makecell{20$\pm$0.0}&\makecell{20}& \makecell{20$\pm$0.0}&\makecell{20}&\makecell{20}&\makecell{20}&\makecell{20$\pm$0.0}&\makecell{20$\pm$0.0}&\makecell{20$\pm$0.0} \\ 
        woodworking-sat11-strips &\makecell{{20$\pm$0.0}}&\makecell{{20}}& \makecell{12$\pm$1.1}&\makecell{{20}}&\makecell{{20}}&\makecell{{20}}&\makecell{{20$\pm$0.0}}&\makecell{{20$\pm$0.0}}&\makecell{{20$\pm$0.0}} \\ 
        zenotravel &\makecell{20$\pm$0.0}&\makecell{20}& \makecell{20$\pm$0.0}&\makecell{20}&\makecell{20}&\makecell{20}&\makecell{20$\pm$0.0}&\makecell{20$\pm$0.0}&\makecell{20$\pm$0.0} \\ \hline
        \textbf{Coverage (1831)} &\makecell{1600$\pm$3.9}& \makecell{1603}& \makecell{1606$\pm$3.9}& \makecell{1535}& \makecell{1590}& \makecell{1626}& \makecell{1641$\pm$1.9} &\makecell{1662$\pm$2.3} &\makecell{{1688$\pm$3.3}}\\ \hline
        \textbf{\% Score (100\%)} &\makecell{83.32\%\\$\pm$0.18}& \makecell{83.23\%}& \makecell{83.51\%\\$\pm$0.27}& \makecell{79.06\%}& \makecell{82.84\%}& \makecell{85.31\%}& \makecell{86.23\%\\$\pm$0.09}& \makecell{87.87\%\\$\pm$0.17} &\makecell{{89.79\%}\\{$\pm$0.22}}\\ \hline
        \textbf{Front-end \% coverage share} &\makecell{-}& \makecell{-}& \makecell{-}& \makecell{-}& \makecell{-}& \makecell{-}&\makecell{97\%} &\makecell{96\%} &\makecell{94\%}\\ \hline
    \end{tabular}
    \vspace{2mm}
    \caption{Comparative performance analysis across the full set of benchmark domains. \textit{\% score} is the average of the \% of instances solved in each domain. \textit{Front-end \% coverage share} refers to the \% of covered instances solved by the BFNoS front-end. Values for BFNoS variants and Approximate-BFWS represent the mean and include the standard deviation across 5 measurements.}
    \label{tab:comparative_performance_large}
\end{table*}

\begin{table*}[tbhp]
    \centering
    \begin{tabular}{|l|c|c|c|c|} \hline
        \rule{0pt}{8pt}Domain & BFNoS- &  BFNoS- & BFNoS- & BFNoS- \\
        & Dual-back$_{MO}$ & LAMA$_{MO}$ & Maidu$_{MO}$ & Maidu-${h^2}$$_{MO}$\\ \hline     
        agricola-sat18-strips &\makecell{15$\pm$0.0}&\makecell{15$\pm$0.0}&\makecell{15$\pm$0.0}&\makecell{15$\pm$0.5} \\ 
        airport &\makecell{47$\pm$0.6}&\makecell{47$\pm$0.6}&\makecell{47$\pm$0.6}&\makecell{47$\pm$0.6} \\
        assembly &\makecell{30$\pm$0.0}&\makecell{30$\pm$0.0}&\makecell{30$\pm$0.0}&\makecell{30$\pm$0.0} \\
        barman-sat14-strips &\makecell{20$\pm$0.0}&\makecell{20$\pm$0.0}&\makecell{20$\pm$0.0}&\makecell{20$\pm$0.0} \\ 
        blocks &\makecell{35$\pm$0.0}&\makecell{35$\pm$0.0}&\makecell{35$\pm$0.0}&\makecell{35$\pm$0.0} \\
        caldera-sat18-adl &\makecell{16$\pm$0.0}&\makecell{17$\pm$0.0}&\makecell{17$\pm$0.0}&\makecell{18$\pm$0.0} \\ 
        cavediving-14-adl &\makecell{8$\pm$0.0}&\makecell{{8$\pm$0.6}}&\makecell{{8$\pm$0.6}}&\makecell{{8$\pm$0.6}} \\
        childsnack-sat14-strips &\makecell{4$\pm$0.5}&\makecell{4$\pm$0.5}&\makecell{4$\pm$0.5}&\makecell{4$\pm$0.6} \\ 
        citycar-sat14-adl &\makecell{20$\pm$0.0}&\makecell{{20$\pm$0.0}}&\makecell{{20$\pm$0.0}}&\makecell{{20$\pm$0.0}} \\ 
        data-network-sat18-strips &\makecell{16$\pm$0.7}&\makecell{16$\pm$0.6}&\makecell{17$\pm$0.6}&\makecell{16$\pm$0.6} \\ 
        depot &\makecell{{22$\pm$0.0}}&\makecell{{22$\pm$0.0}}&\makecell{{22$\pm$0.0}}&\makecell{{22$\pm$0.0}} \\ 
        driverlog &\makecell{20$\pm$0.0}&\makecell{20$\pm$0.0}&\makecell{20$\pm$0.0}&\makecell{20$\pm$0.0} \\
        elevators-sat11-strips &\makecell{20$\pm$0.0}&\makecell{20$\pm$0.0}&\makecell{20$\pm$0.0}&\makecell{20$\pm$0.0} \\ 
        flashfill-sat18-adl &\makecell{17$\pm$0.5}&\makecell{{16$\pm$0.0}}&\makecell{15$\pm$0.5}&\makecell{15$\pm$0.5} \\
        floortile-sat14-strips &\makecell{{2$\pm$0.0}}&\makecell{2$\pm$0.0}&\makecell{2$\pm$0.0}&\makecell{{20$\pm$0.0}} \\ 
        folding &\makecell{{9$\pm$0.0}}&\makecell{10$\pm$1.0}&\makecell{10$\pm$0.6}&\makecell{9$\pm$0.6} \\
        freecell &\makecell{{80$\pm$0.0}}&\makecell{{80$\pm$0.0}}&\makecell{{80$\pm$0.}0}&\makecell{{80$\pm$0.0}} \\ 
        ged-sat14-strips &\makecell{20$\pm$0.0}&\makecell{20$\pm$0.0}&\makecell{20$\pm$0.0}&\makecell{20$\pm$0.0} \\
        grid &\makecell{5$\pm$0.0}&\makecell{5$\pm$0.0}&\makecell{5$\pm$0.0}&\makecell{5$\pm$0.0} \\ 
        gripper &\makecell{20$\pm$0.0}&\makecell{20$\pm$0.0}&\makecell{20$\pm$0.0}&\makecell{20$\pm$0.0} \\ 
        hiking-sat14-strips &\makecell{{20$\pm$0.0}}&\makecell{{20$\pm$0.0}}&\makecell{{20$\pm$0.0}}&\makecell{{20$\pm$0.0}} \\ 
        labyrinth &\makecell{15$\pm$0.5}&\makecell{15$\pm$0.5}&\makecell{15$\pm$0.5}&\makecell{15$\pm$0.5} \\ 
        logistics00 &\makecell{28$\pm$0.0}&\makecell{28$\pm$0.0}&\makecell{28$\pm$0.0}&\makecell{28$\pm$0.0} \\ 
        maintenance-sat14-adl &\makecell{17$\pm$0.0}&\makecell{{17$\pm$0.0}}&\makecell{{17$\pm$0.0}}&\makecell{{17$\pm$0.0}} \\ 
        miconic &\makecell{150$\pm$0.0}&\makecell{150$\pm$0.0}&\makecell{150$\pm$0.0}&\makecell{150$\pm$0.0} \\ 
        movie &\makecell{30$\pm$0.0}&\makecell{30$\pm$0.0}&\makecell{30$\pm$0.0}&\makecell{30$\pm$0.0} \\ 
        mprime &\makecell{35$\pm$0.0}&\makecell{35$\pm$0.0}&\makecell{35$\pm$0.0}&\makecell{35$\pm$0.0} \\ 
        mystery &\makecell{{19$\pm$0.0}}&\makecell{{19$\pm$0.0}}&\makecell{{19$\pm$0.0}}&\makecell{{19$\pm$0.0}} \\ 
        nomystery-sat11-strips &\makecell{19$\pm$0.0}&\makecell{{14$\pm$0.8}}&\makecell{17$\pm$0.0}&\makecell{17$\pm$0.0} \\ 
        nurikabe-sat18-adl &\makecell{16$\pm$0.5}&\makecell{16$\pm$0.6}&\makecell{16$\pm$0.6}&\makecell{{16$\pm$0.6}} \\ 
        openstacks-sat14-strips &\makecell{20$\pm$0.0}&\makecell{20$\pm$0.0}&\makecell{20$\pm$0.0}&\makecell{20$\pm$0.0} \\ 
        organic-synthesis-split-sat18-strips &\makecell{{11$\pm$0.7}}&\makecell{14$\pm$0.6}&\makecell{{13$\pm$0.5}}&\makecell{{13$\pm$1.3}} \\ 
        parcprinter-sat11-strips &\makecell{{20$\pm$0.0}}&\makecell{{20$\pm$0.0}}&\makecell{{20$\pm$0.0}}&\makecell{{20$\pm$0.0}} \\ 
        parking-sat14-strips &\makecell{20$\pm$0.0}&\makecell{20$\pm$0.0}&\makecell{20$\pm$0.0}&\makecell{20$\pm$0.0} \\ 
        pathways &\makecell{30$\pm$0.0}&\makecell{{26$\pm$0.9}}&\makecell{27$\pm$0.8}&\makecell{26$\pm$0.9} \\ 
        pegsol-sat11-strips &\makecell{20$\pm$0.0}&\makecell{20$\pm$0.0}&\makecell{20$\pm$0.0}&\makecell{20$\pm$0.0} \\ 
        pipesworld-notankage &\makecell{50$\pm$0.0}&\makecell{50$\pm$0.0}&\makecell{50$\pm$0.0}&\makecell{50$\pm$0.0} \\ 
        pipesworld-tankage &\makecell{43$\pm$1.1}&\makecell{42$\pm$1.5}&\makecell{42$\pm$1.5}&\makecell{42$\pm$1.5} \\ 
        psr-small &\makecell{50$\pm$0.0}&\makecell{50$\pm$0.0}&\makecell{50$\pm$0.0}&\makecell{50$\pm$0.0} \\ 
        quantum-layout &\makecell{20$\pm$0.0}&\makecell{20$\pm$0.0}&\makecell{20$\pm$0.0}&\makecell{20$\pm$0.0} \\ 
        recharging-robots &\makecell{14$\pm$0.6}&\makecell{{14$\pm$0.6}}&\makecell{{14$\pm$0.6}}&\makecell{{14$\pm$0.6}} \\ 
        ricochet-robots &\makecell{20$\pm$0.0}&\makecell{{20$\pm$0.6}}&\makecell{{20$\pm$0.0}}&\makecell{{20$\pm$0.0}} \\ 
        rovers &\makecell{{40$\pm$0.5}}&\makecell{40$\pm$0.5}&\makecell{{40$\pm$0.5}}&\makecell{40$\pm$0.5} \\ 
        rubiks-cube &\makecell{{5$\pm$0.0}}&\makecell{5$\pm$0.0}&\makecell{{5$\pm$0.0}}&\makecell{5$\pm$0.0} \\ 
        satellite &\makecell{{33$\pm$0.6}}&\makecell{34$\pm$1.1}&\makecell{35$\pm$0.0}&\makecell{34$\pm$1.0} \\ 
        scanalyzer-sat11-strips &\makecell{20$\pm$0.0}&\makecell{20$\pm$0.0}&\makecell{20$\pm$0.0}&\makecell{20$\pm$0.0} \\ 
        schedule &\makecell{{148$\pm$1.3}}&\makecell{148$\pm$1.3}&\makecell{{148$\pm$1.3}}&\makecell{{148$\pm$1.3}} \\ 
        settlers-sat18-adl &\makecell{{12$\pm$0.8}}&\makecell{16$\pm$0.9}&\makecell{16$\pm$0.9}&\makecell{16$\pm$0.6} \\ 
        slitherlink &\makecell{5$\pm$0.0}&\makecell{{4$\pm$0.7}}&\makecell{4$\pm$0.7}&\makecell{4$\pm$0.7} \\ 
        snake-sat18-strips &\makecell{20$\pm$0.0}&\makecell{{20$\pm$0.0}}&\makecell{{20$\pm$0.0}}&\makecell{{20$\pm$0.0}} \\ 
        sokoban-sat11-strips &\makecell{{15$\pm$0.6}}&\makecell{19$\pm$0.5}&\makecell{19$\pm$0.5}&\makecell{{20$\pm$0.0}} \\ 
        spider-sat18-strips &\makecell{18$\pm$0.0}&\makecell{{18$\pm$0.0}}&\makecell{{18$\pm$0.0}}&\makecell{{18$\pm$0.0}} \\ 
        storage &\makecell{29$\pm$0.6}&\makecell{29$\pm$0.6}&\makecell{29$\pm$0.5}&\makecell{29$\pm$0.6} \\ 
        termes-sat18-strips &\makecell{10$\pm$0.5}&\makecell{11$\pm$0.8}&\makecell{11$\pm$0.5}&\makecell{11$\pm$0.5} \\ 
        tetris-sat14-strips &\makecell{20$\pm$0.0}&\makecell{{20$\pm$0.0}}&\makecell{{20$\pm$0.0}}&\makecell{{20$\pm$0.0}} \\ 
        thoughtful-sat14-strips &\makecell{20$\pm$0.0}&\makecell{{20$\pm$0.0}}&\makecell{{20$\pm$0.0}}&\makecell{{20$\pm$0.0}} \\ 
        tidybot-sat11-strips &\makecell{{20$\pm$0.0}}&\makecell{{20$\pm$0.0}}&\makecell{{20$\pm$0.0}}&\makecell{{20$\pm$0.0}} \\ 
        tpp &\makecell{30$\pm$0.0}&\makecell{30$\pm$0.0}&\makecell{30$\pm$0.0}&\makecell{30$\pm$0.0} \\
        transport-sat14-strips &\makecell{20$\pm$0.0}&\makecell{{20$\pm$0.0}}&\makecell{{20$\pm$0.0}}&\makecell{{20$\pm$0.0}} \\
        trucks-strips &\makecell{{18$\pm$0.5}}&\makecell{16$\pm$0.6}&\makecell{19$\pm$0.5}&\makecell{20$\pm$0.9} \\ 
        visitall-sat14-strips &\makecell{20$\pm$0.0}&\makecell{20$\pm$0.0}&\makecell{20$\pm$0.0}&\makecell{20$\pm$0.0} \\ 
        woodworking-sat11-strips &\makecell{{20$\pm$0.0}}&\makecell{{20$\pm$0.0}}&\makecell{{20$\pm$0.0}}&\makecell{{20$\pm$0.0}} \\ 
        zenotravel &\makecell{20$\pm$0.0}&\makecell{20$\pm$0.0}&\makecell{20$\pm$0.0}&\makecell{20$\pm$0.0} \\ \hline
        \textbf{Coverage (1831)} & \makecell{1636$\pm$3.3}& \makecell{1638$\pm$4.9} &\makecell{1643$\pm$3.7} &\makecell{{1662$\pm$4.7}}\\ \hline
        \textbf{\% Score (100\%)} & \makecell{85.90\%$\pm$0.14}& \makecell{86.13\%$\pm$0.26}& \makecell{86.45\%$\pm$0.17} &\makecell{{88.02\%}{$\pm$0.21}}\\ \hline
        \textbf{Front-end \% coverage share} & \makecell{97\%}&\makecell{97\%} &\makecell{97\%} &\makecell{96\%}\\ \hline
    \end{tabular}
    \vspace{2mm}
    \caption{Comparative performance analysis across the full set of benchmark domains of `memory-threshold-only' dual-configuration BFNoS variants. \textit{\% score} is the average of the \% of instances solved in each domain. \textit{Front-end \% coverage share} refers to the \% of covered instances solved by the BFNoS front-end. Coverage values represent the mean and include the standard deviation across 5 measurements.}
    \label{tab:comparative_performance_large_2}
\end{table*}

\end{document}